\documentclass[10pt,twocolumn,letterpaper]{article}

\usepackage{cvpr}              

\usepackage{graphicx}
\usepackage{amsmath}
\usepackage{amssymb}
\usepackage{booktabs}
\usepackage{multirow}
\usepackage{color,xcolor}
\usepackage[accsupp]{axessibility}  

\makeatletter
\@namedef{ver@everyshi.sty}{}
\makeatother
\usepackage{tikz}
\usetikzlibrary{arrows,positioning,shapes.geometric}

\usepackage[mathscr]{eucal}
\usepackage{epsfig,epsf,psfrag}
\usepackage{amssymb,amsmath,amsfonts,latexsym}
\usepackage{amsmath,graphicx,bm,xcolor,url}
\usepackage{fixltx2e}
\usepackage{array}
\usepackage{verbatim}
\usepackage{bm}
\usepackage{algpseudocode}
\usepackage{algorithm}
\usepackage{verbatim}
\usepackage{textcomp}
\usepackage{mathrsfs}
\usepackage{epstopdf}
\usepackage{relsize}
 \usepackage{amsthm}

\catcode`~=11 \def\UrlSpecials{\do\~{\kern -.15em\lower .7ex\hbox{~}\kern .04em}} \catcode`~=13 

\allowdisplaybreaks[3]

\newcommand{\calD}{\mathcal{D}}
\newcommand{\calE}{\mathcal{E}}

\newcommand{\calK}{\mathcal{K}}

\newcommand{\calN}{\mathcal{N}}

\newcommand{\calP}{\mathcal{P}}

\newcommand{\calR}{\mathcal{R}}
\newcommand{\calS}{\mathcal{S}}

\newcommand{\ba}{\mathbf{a}}
\newcommand{\bA}{\mathbf{A}}

\newcommand{\be}{\mathbf{e}}

\newcommand{\bI}{\mathbf{I}}

\newcommand{\bM}{\mathbf{M}}

\newcommand{\bP}{\mathbf{P}}

\newcommand{\bs}{\mathbf{s}}

\newcommand{\bw}{\mathbf{w}}

\newcommand{\bx}{\mathbf{x}}

\newcommand{\by}{\mathbf{y}}

\newcommand{\bz}{\mathbf{z}}

\newcommand{\bbE}{\mathbb{E}}

\newcommand{\bbN}{\mathbb{N}}

\newcommand{\bbP}{\mathbb{P}}

\newcommand{\bbR}{\mathbb{R}}

\DeclareMathAlphabet{\mathbsf}{OT1}{cmss}{bx}{n}
\DeclareMathAlphabet{\mathssf}{OT1}{cmss}{m}{sl}

\DeclareSymbolFont{bsfletters}{OT1}{cmss}{bx}{n}  
\DeclareSymbolFont{ssfletters}{OT1}{cmss}{m}{n}
\DeclareMathSymbol{\bsfGamma}{0}{bsfletters}{'000}
\DeclareMathSymbol{\ssfGamma}{0}{ssfletters}{'000}
\DeclareMathSymbol{\bsfDelta}{0}{bsfletters}{'001}
\DeclareMathSymbol{\ssfDelta}{0}{ssfletters}{'001}
\DeclareMathSymbol{\bsfTheta}{0}{bsfletters}{'002}
\DeclareMathSymbol{\ssfTheta}{0}{ssfletters}{'002}
\DeclareMathSymbol{\bsfLambda}{0}{bsfletters}{'003}
\DeclareMathSymbol{\ssfLambda}{0}{ssfletters}{'003}
\DeclareMathSymbol{\bsfXi}{0}{bsfletters}{'004}
\DeclareMathSymbol{\ssfXi}{0}{ssfletters}{'004}
\DeclareMathSymbol{\bsfPi}{0}{bsfletters}{'005}
\DeclareMathSymbol{\ssfPi}{0}{ssfletters}{'005}
\DeclareMathSymbol{\bsfSigma}{0}{bsfletters}{'006}
\DeclareMathSymbol{\ssfSigma}{0}{ssfletters}{'006}
\DeclareMathSymbol{\bsfUpsilon}{0}{bsfletters}{'007}
\DeclareMathSymbol{\ssfUpsilon}{0}{ssfletters}{'007}
\DeclareMathSymbol{\bsfPhi}{0}{bsfletters}{'010}
\DeclareMathSymbol{\ssfPhi}{0}{ssfletters}{'010}
\DeclareMathSymbol{\bsfPsi}{0}{bsfletters}{'011}
\DeclareMathSymbol{\ssfPsi}{0}{ssfletters}{'011}
\DeclareMathSymbol{\bsfOmega}{0}{bsfletters}{'012}
\DeclareMathSymbol{\ssfOmega}{0}{ssfletters}{'012}

\theoremstyle{plain}
\newtheorem{theorem}{Theorem} 
\newtheorem{lemma}{Lemma}

\newtheorem{corollary}{Corollary}

\newcommand{\qednew}{\nobreak \ifvmode \relax \else
      \ifdim\lastskip<1.5em \hskip-\lastskip
      \hskip1.5em plus0em minus0.5em \fi \nobreak
      \vrule height0.75em width0.5em depth0.25em\fi}

\usepackage[pagebackref,breaklinks,colorlinks]{hyperref}

\usepackage[capitalize]{cleveref}

\crefname{section}{Sec.}{Secs.}
\Crefname{section}{Section}{Sections}
\Crefname{table}{Table}{Tables}
\crefname{table}{Tab.}{Tabs.}

\def\cvprPaperID{7082} 
\def\confName{CVPR}
\def\confYear{2022}

\begin{document}

\title{Non-Iterative Recovery from Nonlinear Observations using Generative Models}

\author{Jiulong Liu\\
LSEC, Institute of Computational
Mathematics \\and Scientific/Engineering
Computing, \\
Academy of Mathematics and System Sciences,\\
Chinese Academy of Sciences, 100190, China\\
{\tt\small jiulongliu@lsec.cc.ac.cn}
\and
Zhaoqiang Liu \thanks{Corresponding author.}\\
Department of Computer Science\\
National University of Singapore\\
{\tt\small dcslizha@nus.edu.sg}
}
\maketitle

\begin{abstract}
   In this paper, we aim to estimate the direction of an underlying signal from its nonlinear observations following the semi-parametric single index model (SIM). Unlike conventional compressed sensing where the signal is assumed to be sparse, we assume that the signal lies in the range of an $L$-Lipschitz continuous generative model with bounded $k$-dimensional inputs. This is mainly motivated by the tremendous success of deep generative models in various real applications. Our reconstruction method is non-iterative (though approximating the projection step may use an iterative procedure) and highly efficient, and it is shown to attain the near-optimal statistical rate of order $\sqrt{(k \log L)/m}$, where $m$ is the number of measurements. We consider two specific instances of the SIM, namely noisy $1$-bit and cubic measurement models, and perform experiments on image datasets to demonstrate the efficacy of our method. In particular, for the noisy $1$-bit measurement model, we show that our non-iterative method significantly outperforms a state-of-the-art iterative method in terms of both accuracy and efficiency.
\end{abstract}


\section{Introduction}
\label{sec:intro}

The basic insight of compressed sensing (CS) is that a high-dimensional sparse signal can be accurately reconstructed from a small number of measurements~\cite{Fou13}. For conventional CS, one aims to recover an $s$-sparse signal $\bx \in \bbR^n$ from linear measurements of the form:
\begin{equation}
 \by = \bA \bx + \bm{\eta},
\end{equation}
where $\bA =[\ba_1,\ba_2,\ldots,\ba_m]^T \in \bbR^{m \times n}$ is the measurement matrix, $\by = [y_1,y_2,\ldots,y_m]^T \in \bbR^m$ is the observed vector, and $\bm{\eta} =[\eta_1,\eta_2,\ldots,\eta_m]^T \in \bbR^m$ is the noise vector. The CS problem has been popular over the last 1--2 decades, and its theoretical properties have been investigated in a significant body of works~\cite{donoho2013information,amelunxen2014living,wen2016sharp,scarlett2019introductory}. For example, under i.i.d.~random Gaussian measurements, it has been shown that an $s$-sparse signal can be accurately and efficiently reconstructed using $O(s\log(n/s))$ samples~\cite{wainwright2009information,arias2012fundamental,candes2013well,scarlett2016limits}.

The conventional CS problem has been extended in a wide variety of directions. Two important ones that we focus on in this paper are (i) considering general nonlinear measurement models, and (ii) assuming that the signal is in the range of a (deep) generative model, instead of being sparse. Both of these settings are practically well-motivated and have attracted sufficient attention in the past years. In the following, we briefly review the background of them.

\subsection{Nonlinear Measurement Models}

While the linear measurement model used in conventional CS can be a good testbed for illustrating conceptual phenomena, in many real problems it may not be justifiable, or even plausible. For example, the binary measurement model used in $1$-bit CS~\cite{boufounos20081} has been of considerable interest because its hardware implementation is low-cost and efficient, and it is also robust to nonlinear distortions. In fact, $1$-bit CS performs even better than conventional CS in certain situations~\cite{laska2012regime}. The limitation of the linear data model motivates the study of general nonlinear measurement models, among which the semi-parametric single index model (SIM) is arguably the most popular one~\cite{horowitz2009semiparametric}. The SIM models the data as
\begin{equation}\label{eq:sim}
 y_i = f_i(\langle \ba_i,\bx\rangle), \quad i \in \{1,2,\ldots,m\},
\end{equation}
where $\ba_i$ are i.i.d.~realizations of a standard Gaussian vector $\ba \sim \calN(\mathbf{0},\bI_n) \in \bbR^n$, with $\ba_i^T$ being the $i$-th row of the measurement matrix $\bA \in \bbR^{m\times n}$; and $f_i\,:\, \bbR\to \bbR$ are i.i.d.~realizations of an {\em unknown} random function $f$, independent of $\ba_i$. The goal is to estimate the signal $\bx$ using the knowledge of $\bA$ and $\by$, despite the unknown nonlinearity $f$. It is well-known that $\bx$ is generally unidentifiable in the SIM since any scaling of $\bx$ can be absorbed into the unknown $f$. Therefore, it is common to impose the identifiability constraint $\|\bx\|_2 = 1$, and only seek to estimate the direction of $\bx$.

Let $y = f(\langle \ba,\bx\rangle)$ be the random variable that corresponds to a single observation. For a SIM and a standard normal random variable $g \sim \calN(0,1)$ that is independent of the nonlinearity $f$, the following parameters are important to characterize the recovery performance of associated reconstruction algorithms:
\begin{equation}\label{eq:mu_def}
        \mu := \bbE[y \langle\ba,\bx\rangle] = \bbE[f(g)g],
       \end{equation}
\begin{equation}\label{eq:xi_sq}
 \xi^2 := \bbE\left[y^2\right] = \bbE\left[f(g)^2\right],
\end{equation}
\begin{equation}\label{eq:rho_sq}
 \rho^2 := \mathrm{Var}[y \langle\ba,\bx\rangle - \mu] = \mathrm{Var}[f(g) g],
\end{equation}
and
\begin{equation}\label{eq:theta_fourth}
 \theta^4 := \mathrm{Var}\left[y^2\right] = \mathrm{Var}\left[f(g)^2\right].
\end{equation}
Note that the parameters $\mu$, $\xi^2$, $\rho^2$ and $\theta^4$ are only used to characterize the recovery performance, and the knowledge of them will not be required for the reconstruction algorithms (since we assume that $f$ is {\em unknown}). Later we will see ({\em cf.}, Section~\ref{sec:main_thm}) that we seek to perform inference on the importance of the components of $\bx$ via estimates of $\mu \bx$, and thus we make the following widely-adopted assumption for the SIM~\cite{plan2017high,plan2016generalized,neykov2016l1,liu2020generalized,eftekhari2021inference}:
\begin{equation}\label{eq:assumption_mu}
 \mu = \bbE[f(g)g] \ne 0.
\end{equation}
We highlight that some popular measurement models such as phase retrieval~\cite{candes2015phase,zhang2017nonconvex} with $f(x) = x^2$ or $f(x) = |x|$ (or the noisy version) are typically beyond the scope of SIM since for these models, $\mu = \bbE[f(g)g] =0$.

For the low-dimensional setting where the number of samples $m$ is larger than the ambient dimension $n$, the SIM has been studied for a long time, dating back to the last century~\cite{han1987non,li1989regression,sherman1993limiting}. In recent years, high-dimensional SIMs have also received much attention, with various papers studying variable selection, estimation and inference mainly under the sparsity assumption~\cite{foster2013variable,ganti2015learning,radchenko2015high,genzel2016high,luo2016forward,neykov2016l1,plan2016generalized,plan2017high,oymak2017fast,cheng2017bs,yang2017high,goldstein2018structured,wei2018structured,pananjady2021single,eftekhari2021inference}. In particular, the authors of~\cite{plan2017high} show that when the signal $\bx$ is contained in $\calK$ for some closed {\em star-shaped}\footnote{A set $\calK$ is called star-shaped if $\lambda \calK \subseteq \calK$ for any $0 \le \lambda \le 1$.} set $\calK \subseteq \bbR^n$, and the observations $y_i$ are {\em sub-Gaussian},\footnote{A random variable $X$ is said to be sub-Gaussian if $\|X\|_{\psi_2} := \sup_{p\ge 1} p^{-1/2} \left(\bbE\left[|X|^p\right]\right)^{1/p} <\infty$.} then the projection of $\frac{1}{m}\bA^T\by$ onto $\calK$ gives an accurate estimate of $\bx$ with high probability provided that the number of samples is sufficiently large. Based on an idea that the nonlinear measurement model may be transformed into a  scaled linear measurement model with an unconventional noise term, the authors of~\cite{plan2016generalized} show that the generalized Lasso approach, which minimizes the {\em linear least-squares} objective over a {\em convex} set $\calK$, is able to return a reliable estimation of the signal in spite of the unknown nonlinearity. However, the range of a Lipschitz continuous generative model (such as a deep neural network), in general, cannot be star-shaped or convex. Moreover, the recovery error bounds in both works~\cite{plan2017high,plan2016generalized} generally exhibit the $m^{-1/4}$ scaling, which is weaker than the typical $m^{-1/2}$ scaling.

\subsection{Inverse Problems using Generative Models}

Recently, motivated by enormous advances in deep generative models in an abundance of real applications, a new perspective has emerged in CS, in which the commonly-made sparsity assumption is replaced by the generative modeling assumption. That is, rather than being sparse, the signal is assumed to lie in the range of a (deep) generative model. In the seminal work~\cite{bora2017compressed}, the authors study CS with generative priors, and characterize the number of random Gaussian linear measurements required for accurate recovery. They also perform extensive numerical experiments on image datasets showing that to reconstruct the signal up to a given accuracy, compared to the sparse prior, using a pre-trained generative prior can reduce the required number of measurements by a large factor such as $5$ to $10$. There has been
a substantial volume of follow-up works of~\cite{bora2017compressed}, including~\cite{rick2017one,van2018compressed,dhar2018modeling,hand2018phase,hand2018global,heckel2019deep,wu2019deep,jalal2020robust,asim2020invertible,ongie2020deep,whang2020compressed,menon2020pulse,jalal2021instance,nguyen2021provable,liu2021towards,liu2021robust}.

In particular, $1$-bit CS with generative priors has been studied in~\cite{liu2020sample,qiu2020robust}, for which the nonlinearity is assumed to be {\em known}. In~\cite{liu2020sample}, the authors provide a near-complete analysis for $1$-bit CS with generative priors, and propose an iterative algorithm that can be thought of as a generative counterpart to the binary iterative hard thresholding algorithm~\cite{jacques2013robust}. The authors of~\cite{qiu2020robust} study $1$-bit CS with ReLU neural network generative models (with no offsets). They propose an empirical risk minimization algorithm, and show that it can faithfully recover bounded target vectors from quantized noisy measurements. Perhaps closest to our work, near-optimal non-uniform recovery guarantees for CS with SIMs and generative priors have been provided in~\cite{wei2019statistical,liu2020generalized}. The authors of~\cite{wei2019statistical} assume that the nonlinear function $f$ is {\em differentiable} and propose estimators via first- and second-order Steins identity based score functions. The differentiability assumption is not satisfied for certain popular nonlinear measurement models such as $1$-bit and other quantized models. The authors of~\cite{liu2020generalized} make the assumption of {\em sub-Gaussian} observations, which encompasses quantized measurement models. They propose a constrained linear least-squares estimator, with the constraint set being the range of a generative model. Both works~\cite{wei2019statistical,liu2020generalized} are primarily theoretical, and neither practical algorithms nor numerical results are provided in these works, even though attaining the estimators may be practically difficult since the corresponding optimization problems are usually highly non-convex.

\subsection{Contributions}
The main contributions of this work are as follows:
\begin{itemize}
 \item We propose a highly efficient non-iterative approach for nonlinear CS with SIMs and generative priors.

 \item We provide near-optimal recovery guarantees for our non-iterative approach. Notably, in our analysis, we do not require the differentiability assumption as in~\cite{wei2019statistical} or the  assumption of sub-Gaussian observations as in~\cite{liu2020generalized}.

 \item To verify the efficacy of our method, we perform a variety of numerical experiments for distinct nonlinear measurement models on image datasets. In particular, for the noisy $1$-bit measurement model, we observe that along with faster computation, our non-iterative approach also leads to more accurate reconstruction compared to the iterative algorithm proposed in~\cite{liu2020sample}, which is the state-of-the-art (SOTA) algorithm for $1$-bit CS with generative priors. In addition, for the noisy cubic measurement model, we observe that our non-iterative approach significantly outperforms several baselines, and performs on par with an iterative approach.
\end{itemize}

We also present Figure~\ref{fig:contributions} to highlight the overall contributions. See~\eqref{eq:f4_assumption},~\eqref{eq:oneshot}, and Section~\ref{sec:exp} for more details.

 \begin{figure}
    \begin{tikzpicture}[font=\small,scale=1.00,>=latex']
        \tikzset{block/.style= {draw, rectangle, align=center,minimum width=2cm,minimum height=1cm},
        }
        \node [block]  (SIM) {  SIM: $\mathbf{y}=f(\bA \bx)$   };
        \node [block, right = 1.0cm of SIM]  (OneShot) {  OneShot:   $\hat{\bx}= \calP_G\left(\frac{1}{m} \bA^{T} \mathbf{y}\right)$};
        \node [coordinate, below = 0.5cm of OneShot] (ADL){};
        \node [coordinate, left = 3.7cm of ADL] (AUL){};
        \node [coordinate, right = 0.7cm of ADL] (BDL){};

        \node [block, minimum height=1.6cm, below = 0.5cm of AUL] (A1){Theoretical guarantees\\ under assumption \\ $\mathbb{E}\left[f(g)^{4}\right] <\infty, g \sim \mathcal{N}(0,1)$};
        \node [block, minimum height=1.6cm, below = 0.5cm of BDL] (B1){Numerical experiments  for \\1) noisy $1$-bit model\\ 2) noisy cubic model};

        \path[draw, ->]
            (SIM) edge (OneShot)
            (OneShot) edge (A1)

            (OneShot) -- (B1);

    \end{tikzpicture}
\caption{An illustration of main contributions.}\label{fig:contributions}
\end{figure}
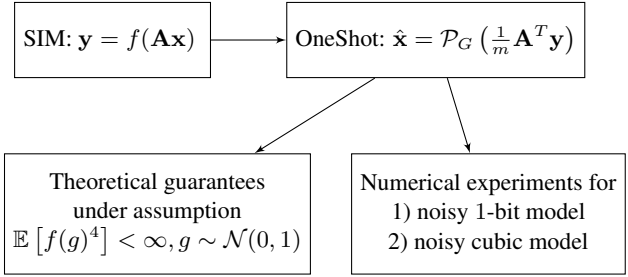
\section{Problem Formulation}

In this section, we provide some auxiliary results and formally formulate the problem we study. Before proceeding, we summarize the notation we use throughout this paper.

\subsection{Notation}
 We use upper and lower case boldface letters to denote matrices and vectors respectively. For any $N \in \bbN$, we use the shorthand notation $[N] = \{1,2,\ldots,N\}$, and we use $\bI_N$ to represent the identity matrix in $\bbR^{N\times N}$. For a matrix $\bM$, let $\|\bM\|_{p,q} = \sup_{\|\bs\|_p=1} \|\bM \bs\|_q$. In particular, $\|\bM\|_{2,2}$ represents the spectral norm of $\bM$. Given two sequences of real values $\{a_i\}$ and $\{b_i\}$, we write $a_i = O(b_i)$ if there exists an absolute constant $C_1$ and a positive integer $i_1$ such that for any  $i>i_1$, $|a_i| \le C_1 b_i$, $a_i = \Omega(b_i)$ if there exists an absolute constant $C_2$ and a positive integer $i_2$ such that for any  $i>i_2$, $|a_i| \ge C_2 b_i$, and $a_i = \Theta(b_i)$ if $a_i = O(b_i)$ and $a_i = \Omega(b_i)$. For any $r > 0$, we denote the radius-$r$ ball in $\bbR^k$ as $B_2^k(r) := \{\bz \in \bbR^k\,:\, \|\bz\|_2 \le r\}$, and we use $\calS^{n-1}:= \{\bs \in \bbR^n\,:\, \|\bs\|_2 =1\}$ to represent the unit sphere in $\bbR^n$. A generative model is a function $G \,:\, \calD \to \bbR^n$ with latent dimension $k$, ambient dimension $n$, and input domain $\calD \subseteq \bbR^k$. For a generative model $G$ and a set $B \subseteq \calD$, we write $G(B) =\{G(\bz)\,:\, \bz \in B\}$. Throughout the following, we focus on the setting that $\calD =B_2^k(r)$ and $k \ll n$. We use $\calR(G)$ to represent the range of $G$, i.e., $\calR(G) = G(B_2^k(r))$.

 \subsection{Setup}

Suppose that the generative model $G\,:\, B_2^k(r)\to \bbR^n$ is $L$-Lipschitz continuous, i.e., $\|G(\bz_1)-G(\bz_2)\|_2 \le L\|\bz_1-\bz_2\|_2$ for any $\bz_1, \bz_2 \in B_2^k(r)$. The Lipschitzness assumption is naturally satisfied by some popular neural network generative models. For example, it is shown in~\cite{bora2017compressed,liu2020sample} that any fully-connected neural network generative model with bounded weights and widely-used activation functions (including Sigmoid, ReLU and Hyperbolic tangent functions) is Lipschitz continuous with the Lipschitz constant being $L = n^{\Theta(d)}$, where $d$ is the depth of the neural network.

The nonlinear observations $y_1,y_2,\ldots,y_m$ are assumed to be generated according to the SIM in~\eqref{eq:sim}, with $\ba_i$ being i.i.d. realizations of $\calN(\mathbf{0},\bI_n)$ and $\bx \in \calS^{n-1}$ being the signal to estimate. We further assume that $\mu \bx \in \calR(G)$, where $\mu$ is a parameter depending on the nonlinearity $f$ and is defined in~\eqref{eq:mu_def}, and $\calR(G)$ refers to the range of $G$. Such an assumption is standard for nonlinear CS with generative priors and has also been made in~\cite{wei2019statistical,liu2020generalized}. In this work, for the nonlinear function $f$, besides the popular assumption $\mu \ne 0$ as in~\eqref{eq:assumption_mu}, we only additionally assume that
\begin{equation}\label{eq:f4_assumption}
 \bbE\left[f(g)^4\right] < \infty,
\end{equation}
where $g \sim \calN(0,1)$ is a standard normal random variable. Under this assumption, the parameters $\mu, \xi^2, \rho^2$ and $\theta^4$ defined in~\eqref{eq:mu_def} to~\eqref{eq:theta_fourth} are all finite. The condition in~\eqref{eq:f4_assumption} holds for quantized measurement models, which do not satisfy the differentiability assumption in~\cite{wei2019statistical}. Moreover,  it does not require $f(g)$ (corresponds to each observation $y_i$) to be sub-Gaussian as assumed in~\cite{liu2020generalized}, thus enables us to deal with more general nonlinear measurement models such as $f(x) = x^3 + \eta$ (the noisy cubic model) or $f(x) = \mathrm{sign}(x) \big(x^2 +1\big) + \eta$, where $\eta$ is a zero-mean random Gaussian noise term.

To reconstruct the direction of the signal $\bx$ from the knowledge of the measurement matrix $\bA \in \bbR^{m \times n}$ and the observed vector $\by =[y_1,y_2,\ldots,y_m]^T \in \bbR^m$ (despite the unknown nonlinearity $f$), we set the estimated vector to be
\begin{equation}\label{eq:oneshot}
 \hat{\bx} = \calP_G\left(\frac{1}{m}\bA^T\by\right),
\end{equation}
where $\calP_G(\cdot)$ is the projection operator onto $\calR(G)$, i.e., $\calP_G(\bs) = \arg\min_{\bw \in \calR(G)} \|\bw - \bs\|_2$ for any $\bs \in \bbR^n$. This can be thought of as a generative counterpart to the methods proposed in~\cite{zhang2014efficient,plan2017high} for sparse priors. We refer to the reconstruction approach corresponding to~\eqref{eq:oneshot} as {\em OneShot} to highlight its non-iterative nature, although approximating the projection step may use iterative procedures such as gradient descent. 


\section{Main Theorem}
\label{sec:main_thm}

We have the following theorem concerning the recovery guarantee for OneShot in~\eqref{eq:oneshot}. Recall that $\mu, \xi^2, \rho^2$ and $\theta^4$ are parameters that are dependent only on the nonlinearity $f$ and are defined in~\eqref{eq:mu_def} to~\eqref{eq:theta_fourth}.

\begin{theorem}\label{thm:main}
 Suppose that the observed vector $\by \in \bbR^m$ is generated from the SIM in~\eqref{eq:sim} with $\ba_i$ being i.i.d. realizations of $\calN(\mathbf{0},\bI_n)$, $f$ satisfying~\eqref{eq:assumption_mu} and~\eqref{eq:f4_assumption}, and $\bx \in \calS^{n-1}\cap \frac{1}{\mu}\calR(G)$. Let $\hat{\bx}$ be calculated from~\eqref{eq:oneshot}. Then, for any $\delta>0$ satisfying $Lr = \Omega(\delta n)$ and $\delta = O\left(\xi\sqrt{\frac{k \log \frac{Lr}{\delta}}{m}}\right)$, we have with probability at least $1-e^{-\Omega\big(k\log\frac{Lr}{\delta}\big)} - \frac{\theta^4}{m\xi^4} - \frac{\rho^2}{\xi^2 k \log \frac{Lr}{\delta}}$ that
 \begin{equation}\label{eq:main_thm_ub}
  \|\hat{\bx}-\mu\bx\|_2 = O\left(\xi\sqrt{\frac{k \log \frac{Lr}{\delta}}{m}}\right).
 \end{equation}
\end{theorem}
Since a typical $d$-layer fully-connected neural network has Lipschitz constant $L = n^{\Theta(d)}$~\cite{bora2017compressed}, the assumption $Lr = \Omega(\delta n)$ is typically satisfied automatically. In addition, from the assumption~\eqref{eq:f4_assumption}, $\xi$ is finite. Then, we have that the upper bound in~\eqref{eq:main_thm_ub} is roughly of order $\sqrt{(k \log L)/m}$, which is naturally conjectured to be near-optimal according to the information-theoretic lower bounds for linear CS with generative priors~\cite{liu2020information,kamath2020power}. Perhaps the major caveat to Theorem~\ref{thm:main} is that it assumes the accurate projection. However, this is a standard assumption in relevant works, e.g., see~\cite{shah2018solving,peng2020solving,liu2022generative}, and in practice both gradient- and GAN-based projections have been shown to be highly effective~\cite{shah2018solving,raj2019gan}.


\subsection{Proof Outline of Theorem~\ref{thm:main}}

The proof of Theorem~\ref{thm:main} is outlined below, with the full details provided in the supplementary material. Define the event 
 \begin{equation}\label{eq:eventE}
  \calE = \left\{\frac{1}{m}\sum_{i=1}^m y_i^2 \le 2\xi^2\right\},
 \end{equation}
where $\xi^2$ is defined in~\eqref{eq:xi_sq}. From Chebyshev's inequality and the definition of $\theta^4$ in~\eqref{eq:theta_fourth}, we have
\begin{equation}\label{eq:cheby_eventE}
 \bbP(\calE^c) \le \frac{\theta^4}{m\xi^4}.
\end{equation}
 Define $\bP := \bx\bx^T$ as the orthogonal projection onto the subspace spanned by $\bx$ and $\bP^{\bot} := \bI_n -\bx\bx^T$ as the orthogonal projection onto the orthogonal complement. Based on standard Gaussian concentration~\cite[Example~2.1]{wainwright2019high}, we have the following important lemma, whose proof is given in the supplementary material.
\begin{lemma}\label{lem:imp_f4}
Conditioned on the event $\calE$, we have that for any $\varepsilon > 0$ and $\bs \in \bbR^n$, with probability $1-e^{-\Omega(\varepsilon)}$,
\begin{equation}\label{eq:lem_imp_eq0}
 \left|\frac{1}{m}\sum_{i=1}^m y_i\left\langle \bP^{\bot}\ba_i,\bs\right\rangle\right| \le \frac{\xi \|\bs\|_2\sqrt{\varepsilon}}{\sqrt{m}}.
\end{equation}
\end{lemma}

We now move on to present the proof outline.
\begin{proof}[Proof Outline of Theorem~\ref{thm:main}]
 Since $\hat{\bx} = \calP_G\big(\frac{1}{m}\bA^T\by\big)$ and $\mu \bx \in \calR(G)$, we have $\big\|\frac{1}{m}\bA^T\by - \hat{\bx}\big\|_2 \le \big\|\frac{1}{m}\bA^T\by - \mu \bx\big\|_2$.
Taking square on both sides, we obtain
\begin{equation}\label{eq:thm_eq1_main}
 \|\hat{\bx} - \mu\bx\|_2^2 \le 2\left\langle \frac{1}{m}\bA^T\by -\mu\bx, \hat{\bx} - \mu\bx \right\rangle.
\end{equation}
Recall that $\bP^{\bot} := \bI_n -\bx\bx^T$. To upper bound the right-hand side of~\eqref{eq:thm_eq1_main}, we decompose $\frac{1}{m}\bA^T\by -\mu\bx$ as
\begin{align}
 & \frac{1}{m}\bA^T\by -\mu\bx  = \frac{1}{m}(\bI_n-\bx\bx^T +\bx\bx^T)\bA^T\by  -\mu\bx \\
 & = \frac{1}{m}\bP^{\bot}\bA^T\by + \left(\frac{1}{m}\bx^T\bA^T\by -\mu\right)\bx.\label{eq:thm_eq2}
\end{align}
Then, we obtain that
\begin{itemize}
 \item the term $\big|\big\langle\frac{1}{m}\bP^{\bot}\bA^T\by, \hat{\bx} - \mu\bx\big\rangle\big|$ can be controlled using~\eqref{eq:cheby_eventE}, Lemma~\ref{lem:imp_f4}, and a chaining argument~\cite{bora2017compressed};
 \item the term $\big|\big\langle\big(\frac{1}{m}\bx^T\bA^T\by -\mu\big)\bx, \hat{\bx}-\mu\bx \big\rangle\big|$ can be controlled using the triangle inequality, and Chebyshev's inequality with the definition of $\rho^2$ in~\eqref{eq:rho_sq}.
\end{itemize}
Combining the two upper bounds and simplifying terms, we obtain the desired result in Theorem~\ref{thm:main}.
\end{proof}


\subsection{Extensions of Theorem~\ref{thm:main}}

We present a corollary extending Theorem~\ref{thm:main} in two directions. Specifically, this corollary shows that we can allow for {\em adversarial noise} that may be dependent on the measurement matrix $\bA$ and the existence of {\em representation error} where $\mu \bx \notin \calR(G)$. It is worth noting that for the generalized Lasso approach considered in~\cite{plan2016generalized,liu2020generalized}, handling representation error is not a simple task and is left open. The proof of Corollary~\ref{coro:first} is given in the supplementary material.

\begin{corollary}\label{coro:first}
 Suppose that the observed vector $\by = [y_1,y_2,\ldots,y_m]^T \in \bbR^m$ satisfies
 \begin{equation}
  \frac{1}{\sqrt{m}} \sqrt{\sum_{i=1}^m \left(y_i - f_i(\langle \ba_i,\bx\rangle)\right)^2} \le \nu
 \end{equation}
for some $\nu \ge 0$, with $\ba_i$ being i.i.d. realizations of $\calN(\mathbf{0},\bI_n)$, $f_i$ being i.i.d.~realizations of $f$, $f$ satisfying~\eqref{eq:assumption_mu} and~\eqref{eq:f4_assumption}, and $\bx \in \calS^{n-1}$. Let $\tilde{\bx} = \calP_G(\mu\bx)$ be the vector in $\calR(G)$ that is closest to $\mu\bx$, and let $\hat{\bx}$ be calculated from~\eqref{eq:oneshot}. Then, for any $\delta>0$ satisfying $Lr = \Omega(\delta n)$ and $\delta = O\left(\xi\sqrt{\frac{k \log \frac{Lr}{\delta}}{m}}\right)$, when $m = \Omega(k\log\frac{Lr}{\delta})$, we have with probability at least $1-e^{-\Omega\big(k\log\frac{Lr}{\delta}\big)} - \frac{\theta^4}{m\xi^4} - \frac{\rho^2}{\xi^2 k \log \frac{Lr}{\delta}}$ that
 \begin{equation}
  \|\hat{\bx}-\mu\bx\|_2 = O\left(\xi\sqrt{\frac{k \log \frac{Lr}{\delta}}{m}} + \nu + \|\tilde{\bx}-\mu\bx\|_2\right).
 \end{equation}
\end{corollary}


\section{Experiments}\label{sec:exp}
The proposed method is evaluated on two special cases of the SIM in~\eqref{eq:sim}, namely a noisy $1$-bit measurement model
\begin{equation}
\label{eqn:1bitm}
 y_i = \mbox{sign}(\langle \ba_i,\bx\rangle+e_i), \quad i \in [m],
\end{equation}
where $e_{i}$ are i.i.d.~realizations of $\mathcal{N}\big(0, \sigma^{2}\big)$, and a noisy cubic measurement model
\begin{equation}
\label{eqn:cubicm}
 y_i =\langle \ba_i,\bx\rangle^{3}+ \eta_i, \quad i \in [m],
\end{equation}
 where $\eta_{i}$ are i.i.d.~realizations $\mathcal{N}\big(0, \sigma^{2}\big)$. Note that representation error is implicitly allowed in our experiments since the image vectors are not exactly contained in the range of the generative model. For simplicity, throughout this section, we do not consider adversarial noise. Experimental results with adversarial noise and the visualization of samples generated from pre-trained generative models are presented in the supplementary material.
 

\subsection{Implementation Details}

The experiments are performed on the MNIST~\cite{lecun1998gradient} and CelebA~\cite{liu2015deep} datasets. The MNIST dataset consists of $60,000$ images of handwritten digits. The size of each image in the MNIST dataset is $28 \times 28$, and thus the ambient dimension is $n = 784$. The CelebA dataset contains more than $200,000$ face images of celebrities. Each input image was cropped to a $64\times64$ RGB image, giving $n =64\times64\times3 = 12288$ inputs per image. The generative model $G$ for the MNIST dataset is set to be a pre-trained variational autoencoder (VAE) model with latent dimension $k = 20$. The encoder and decoder are both fully connected neural networks with two hidden layers, with the architecture being $20-500-500-784$. The VAE is trained by the Adam optimizer with a mini-batch size of $100$ and a learning rate of $0.001$  using the original training set of MNIST.

For the CelebA dataset, we choose the DCGAN~\cite{radford2015unsupervised,kim2017tensorflow} for the generative model $G$. The architecture of the DCGAN follows that in \cite{kim2017tensorflow} and the dimension of the input vector is set to be $k = 100$, with each entry being independently drawn from the standard normal distribution. We use the same training setup as in \cite{bora2017compressed} to train the DCGAN on the training set of CelebA. Images that are selected from the testing sets (unseen by the pre-trained generative models) of the MNIST and CelebA datasets are used to generate $1$-bit and cubic observations based on \eqref{eqn:1bitm} and \eqref{eqn:cubicm} respectively, with more details being listed in Table \ref{tab:meas}.

To approximate the projection step $\calP_G(\cdot)$, we utilize the following two methods: 1) A gradient descent method that is performed using the Adam optimizer with $100$ steps and a learning rate of $0.1$. Such an {\em iterative} method is also used in~\cite{liu2020sample,shah2018solving,peng2020solving,liu2022generative}. 2) A GAN-based projection method~\cite{raj2019gan} that is {\em non-iterative} and much faster, for which we follow the settings in~\cite{raj2019gan} to train the GAN models used in our experiments.

\begin{table}[h]
\caption{\label{tab:meas}The parameters of the measurement models}
\begin{tabular}{|ccc|}
\hline
\multicolumn{3}{|c|}{$1$-bit measurements}                                                   \\ \hline
\multicolumn{1}{|c|}{Dataset}    & \multicolumn{1}{c|}{$\sigma$}               & $m$               \\ \hline
\multicolumn{1}{|c|}{MNIST}  & \multicolumn{1}{c|}{0.1, 0.5, 1, 5}    & 25, 50, 100, 200, 400 \\ \hline
\multicolumn{1}{|c|}{CelebA} & \multicolumn{1}{c|}{0.01,0.05,0.1,0.5} & 4000,6000,10000,15000  \\ \hline
\multicolumn{3}{|c|}{cubic measurements}                                                   \\ \hline
\multicolumn{1}{|c|}{Dataset}    & \multicolumn{1}{c|}{$\sigma$ }            & $m$               \\ \hline
\multicolumn{1}{|c|}{MNIST}  & \multicolumn{1}{c|}{0.1, 0.5, 1, 5}    & 25, 50, 100, 200, 400 \\ \hline
\multicolumn{1}{|c|}{CelebA} & \multicolumn{1}{c|}{0.01,0.05,0.1,0.5} & 4000,6000,10000,15000  \\ \hline
\end{tabular}
\end{table}

We perform the recovery tasks for the two nonlinear measurement models described in~\eqref{eqn:1bitm} and~\eqref{eqn:cubicm} using our proposed non-iterative method as in~\eqref{eq:oneshot} (denoted by~\texttt{OneShot} when using iterative projection, and~\texttt{OneShotF} when using faster non-iterative projection), with comparison to some sparsity-based methods, and the method proposed in \cite{bora2017compressed} (denoted by \texttt{CSGM}), as well as some generative model based projected iterative methods. For the sparse recovery with MNIST, we use Lasso~\cite{Tib96} on the images in the image domain with the shrinkage parameter setting to be $0.1$ (denoted by \texttt{Lasso}). For the sparse recovery with CelebA, we use Lasso on the images in the wavelet domain using 2D Daubechies-$1$ Wavelet Transform with the shrinkage parameter setting to be $0.00001$ (denoted by \texttt{Lasso-W}). For the projected iterative method with $1$-bit measurements, we use the method proposed in \cite{liu2020sample} (\underline{B}inary \underline{I}terative (\underline{F}ast) \underline{P}rojected \underline{G}radient method, denoted by~\texttt{BIPG} when using iterative projection, and~\texttt{BIFPG} when using faster non-iterative projection), which is the SOTA method for $1$-bit CS with generative priors, using the same pre-trained generative model as those described above. The corresponding formula is as follows:
\begin{equation}
 \label{eq:bipg}
\mathbf{x}^{(t+1)}=\mathcal{P}_{G}\left(\mathbf{x}^{(t)}+\lambda \mathbf{A}^{T}\left(\mathbf{y}-\operatorname{sign}\left(\mathbf{A} \mathbf{x}^{(t)}\right)\right)\right).
\end{equation}
 For the projected iterative method with cubic measurements, since there is no existing method specifically designed for this case, we compare with the method proposed in \cite{shah2018solving,peng2020solving} (denoted by~\texttt{PGD} when using iterative projection, and~\texttt{FPGD} when using faster non-iterative projection), although it is initially designed for {\em linear} CS with generative priors. We also use the same pre-trained generative model as those described above. The corresponding formula is as follows:
 \begin{equation}
 \label{eq:pgd}
\mathbf{x}^{(t+1)}=\mathcal{P}_{G}\left(\mathbf{x}^{(t)}+\lambda \mathbf{A}^{T}\left(\mathbf{y}- \mathbf{A} \mathbf{x}^{(t)}\right)\right).
\end{equation}
For~\texttt{BIPG} (or~\texttt{BIFPG}) and~\texttt{PGD} (or~\texttt{FPGD}), we set the step size as $\lambda = 1/m$, the initial vector as $\bx^{(0)} = \mathbf{0}$, and the total number of iterations as $T = 30$. 

All experiments are run using Python 3.6 and TensorFlow 1.5.0, with a NVIDIA GeForce GTX 1080 Ti 11GB GPU. To reduce the impact of local minima, we perform $10$ random restarts, and choose the best among these.  The cosine similarity refers to the inner product between the signal $\bx$ and the normalized output vector of each recovery method, and it is averaged over both the testing images and these $10$ random restarts.


\subsection{Recovery Results from $1$-bit Measurements}
The reconstructed images of the MNIST dataset from $1$-bit measurements are shown in Figure~\ref{fig:mnist_1bit}, where we consider two settings with $\sigma = 1.0, m = 200$ and $\sigma = 0.1, m = 400$. In addition, we provide quantitative comparisons according to cosine similarity. To illustrate the effect of the sample size $m$, the cosine similarity in terms of  $m \in \{25,50,100,200,400\}$ for MNIST reconstruction is plotted in Figure~\ref{fig:mnist_1bit_cs}a, with fixing $\sigma = 1$. In addition, to illustrate the effect of the noise level $\sigma$, the cosine similarity in terms of $\sigma \in \{0.1, 0.5, 1, 5\}$ for MNIST reconstruction is plotted in Figure \ref{fig:mnist_1bit_cs}b, with fixing $m = 200$. We observe the following from Figures~\ref{fig:mnist_1bit} and~\ref{fig:mnist_1bit_cs}:
\begin{itemize}
 \item \texttt{Lasso} and~\texttt{CSGM} attain poor reconstructions.
 \item \texttt{OneShot} and~\texttt{BIPG} significantly outperform all other methods, with the reconstruction performance of~\texttt{OneShot} being slightly better than~\texttt{BIPG}, though~\texttt{OneShot} is non-iterative and performs much faster than~\texttt{BIPG} ({\em cf.} Table~\ref{tab:timeelapsed}).
 \item \texttt{OneShot} outperforms \texttt{OneShotF} and~\texttt{BIPG} outperforms~\texttt{BIFPG}, which amount to showing that at least for the MNIST dataset, the faster computation of the GAN-based projection step comes at the price of worse reconstruction.
\end{itemize}

The reconstructed images of the CelebA dataset from $1$-bit measurements are shown in Figure~\ref{fig:celebA_1bit}, where we consider two settings $\sigma = 0.01, m = 4000$ and $\sigma = 0.05, m = 10000$. To illustrate the effect of the sample size $m$, the cosine similarity in terms of $m \in \{4000, 6000,10000,15000\}$ for CelebA reconstruction\footnote{Note that for $1$-bit CS, it is very practical to set $m > n = 12288$ since $1$-bit measurements can be taken at extremely high rates~\cite{zhang2014efficient}.} is plotted in Figure \ref{fig:celebA_1bit_cs}a, with fixing $\sigma = 0.01$. In addition, to illustrate the effect of the noise level $\sigma$, the cosine similarity in terms of $\sigma \in \{0.01, 0.05, 0.1, 0.5\}$ for CelebA reconstruction is plotted in Figure \ref{fig:celebA_1bit_cs}b, with fixing $m = 4000$. We observe the following from Figures~\ref{fig:celebA_1bit} and~\ref{fig:celebA_1bit_cs}:
\begin{itemize}
 \item \texttt{Lasso-W} almost fails to recover the images, although the cosine similarities corresponding to~\texttt{Lasso-W} are not small.
 \item \texttt{CSGM}, \texttt{BIPG}, and \texttt{BIFPG} lead to inferior reconstruction performance.
 \item Our proposed methods \texttt{OneShot} and \texttt{OneShotF} obtain much better reconstruction compared to all other methods. It is worth noting that the non-iterative approach~\texttt{OneShot} (or~\texttt{OneShotF}) significantly outperforms the projected iterative approach~\texttt{BIPG} (or~\texttt{BIFPG}). While this is a bit counter-intuitive, similar results concerning sparse priors have been reported in~\cite[Figures $1$ to $3$]{zhang2014efficient}, showing that for synthetic data, a non-iterative approach leads to better recovery performance when compared with a sparse counterpart to~\texttt{BIPG}.
\end{itemize}


\begin{figure}[htp]
\begin{center}
\begin{tabular}{p{3.75cm}<{\centering}p{3.75cm}<{\centering}}
\includegraphics[width=0.23\textwidth]{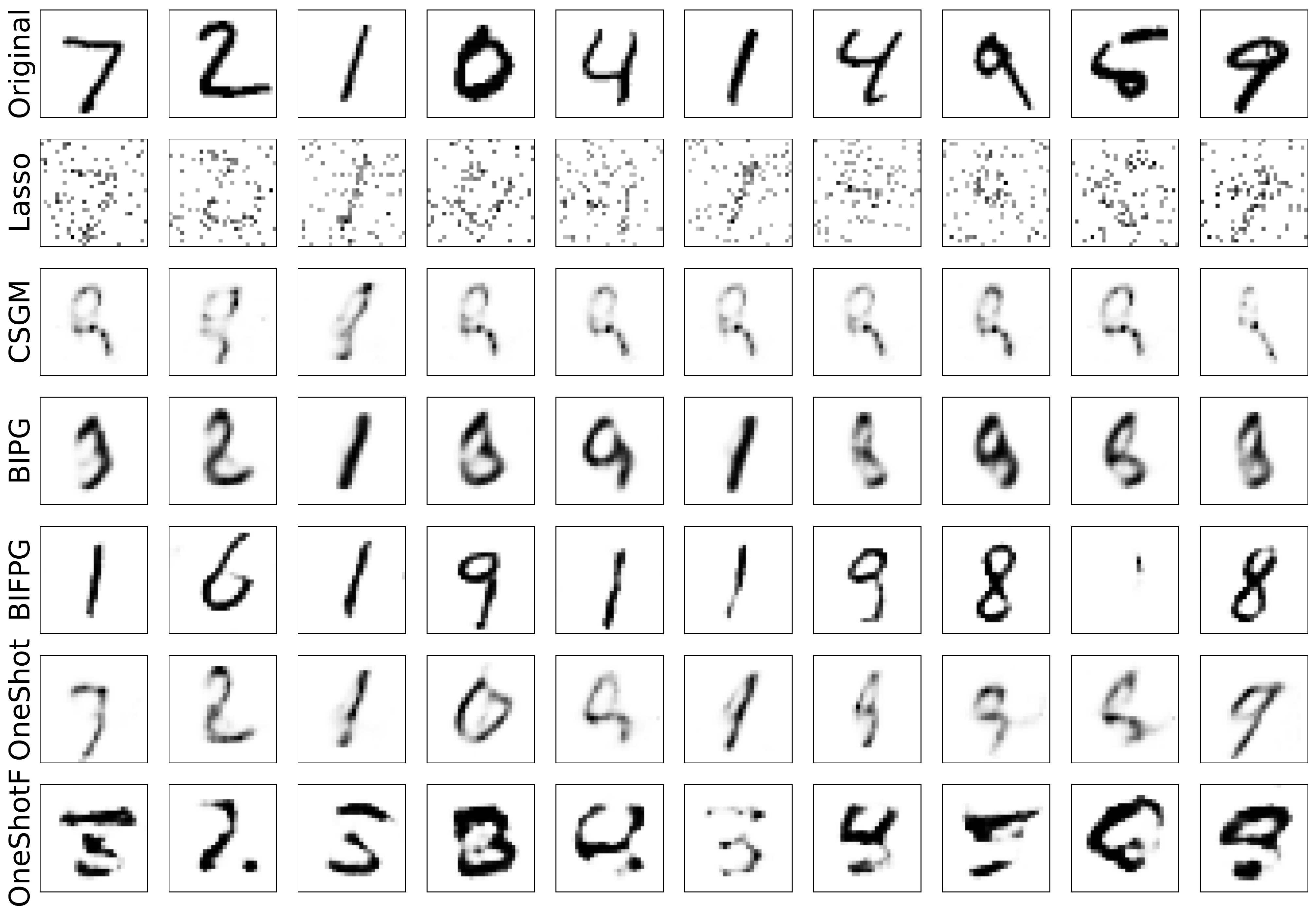} & \includegraphics[width=0.23\textwidth]{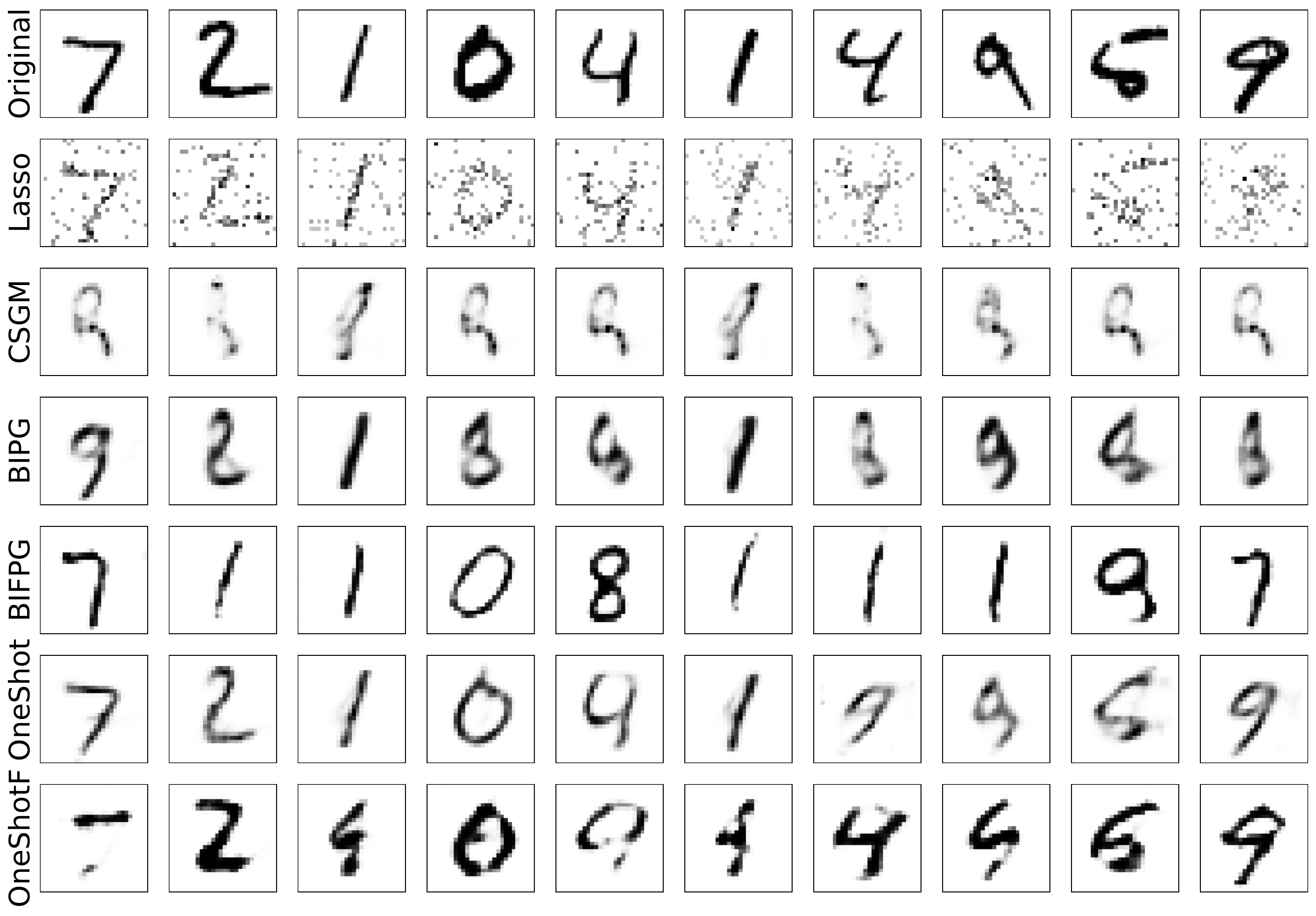}\\
{\small (a) $\sigma = 1.0 $   and $m = 200$}& {\small (b) $\sigma = 0.1 $  and $m = 400$}\\
\end{tabular}
\caption{Examples of reconstructed images from $1$-bit measurements on the MNIST images.}
\label{fig:mnist_1bit}
\end{center}
\end{figure}


\begin{figure}[htp]
\begin{center}
\begin{tabular}{p{3.75cm}<{\centering}p{3.75cm}<{\centering}}
\includegraphics[width=0.23\textwidth]{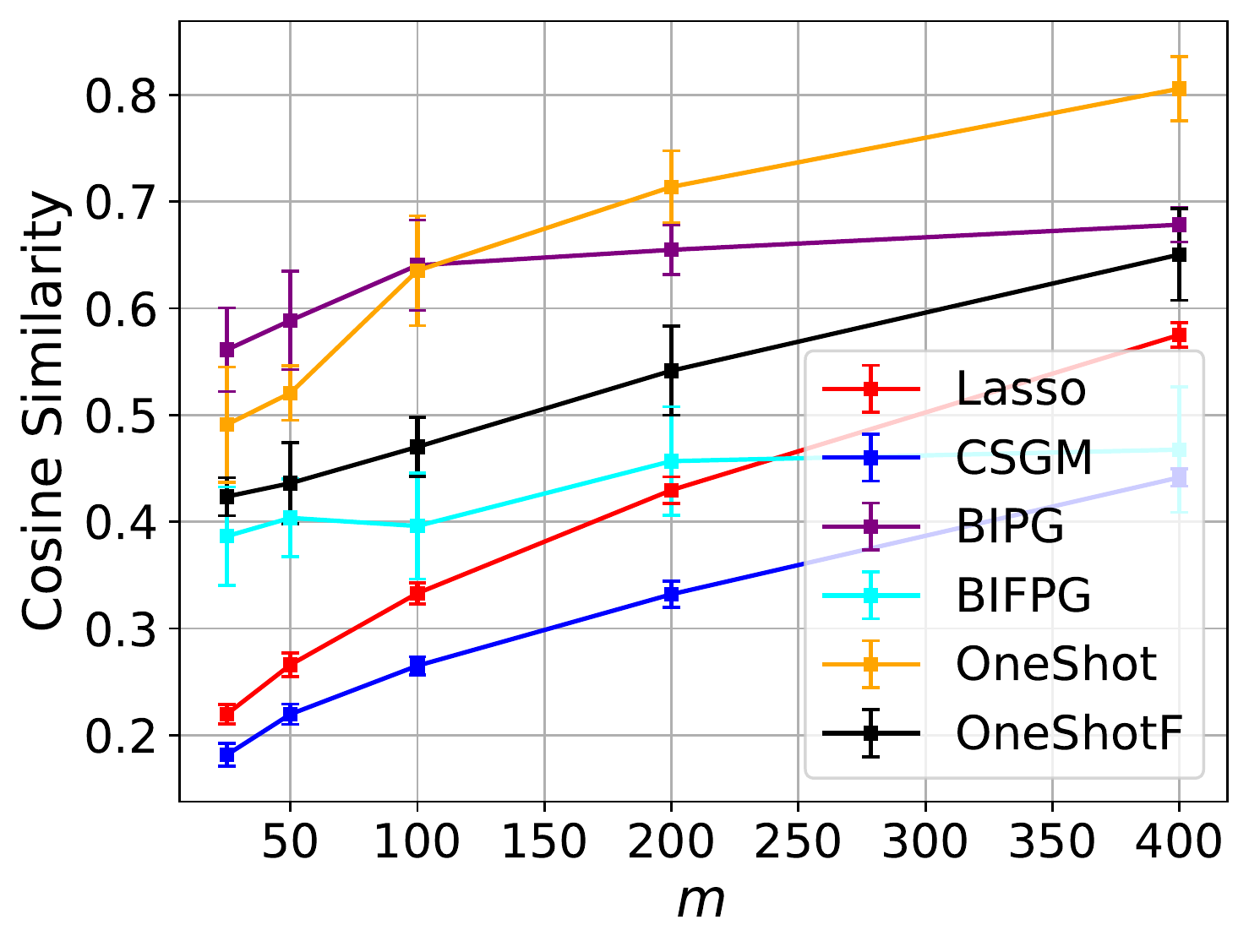} & \includegraphics[width=0.23\textwidth]{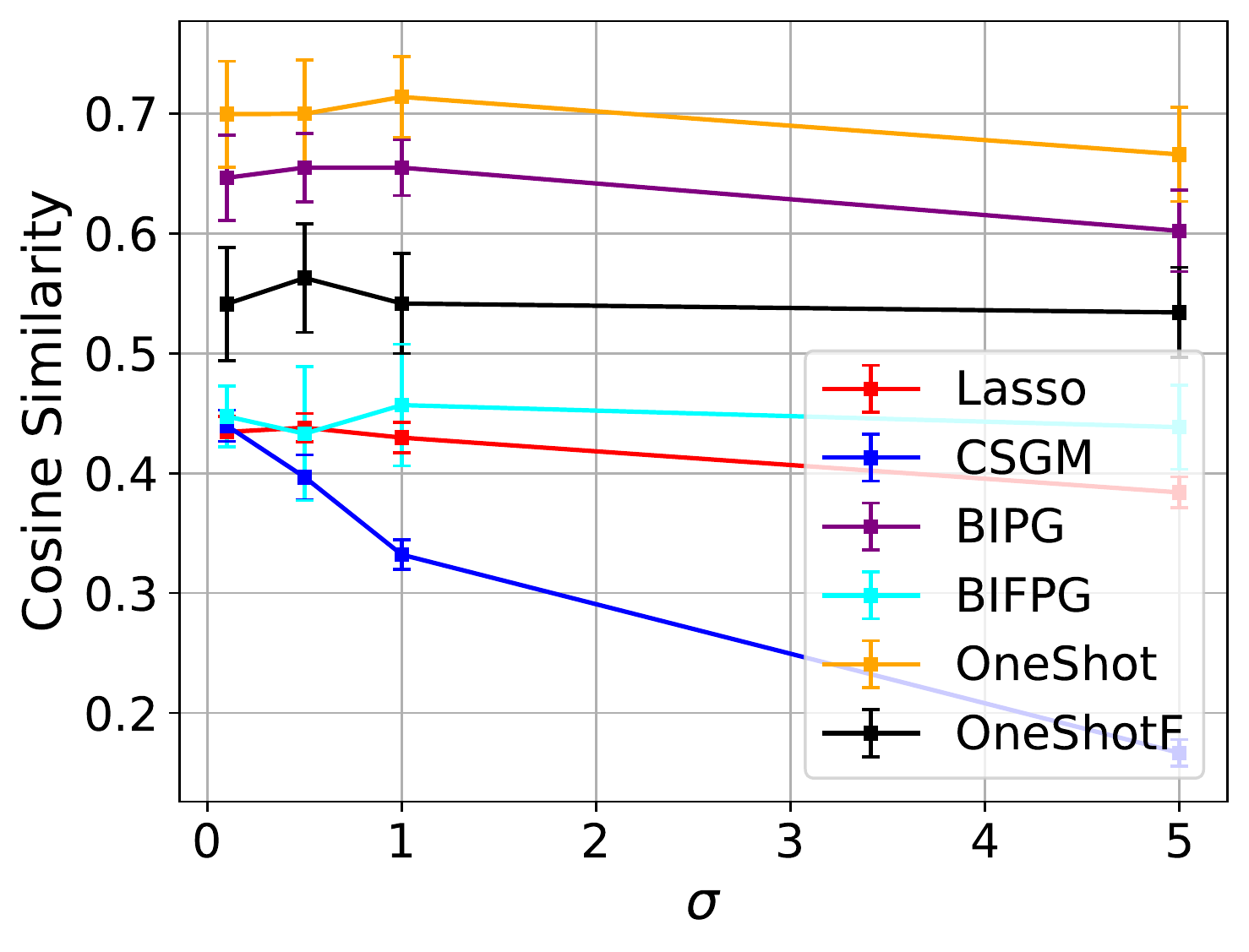}\\
{\small (a) Fixing $\sigma = 1.0$  and varying $m$  }& {\small (b) Fixing $m = 200$} and varying $\sigma$  \\
\end{tabular}
\caption{Quantitative comparisons according to the cosine similarity for $1$-bit measurements on MNIST images.}
\label{fig:mnist_1bit_cs}
\end{center}
\end{figure}

\begin{figure}[htp]
\begin{center}
\begin{tabular}{p{7.5cm}<{\centering}}
\includegraphics[width=0.46\textwidth]{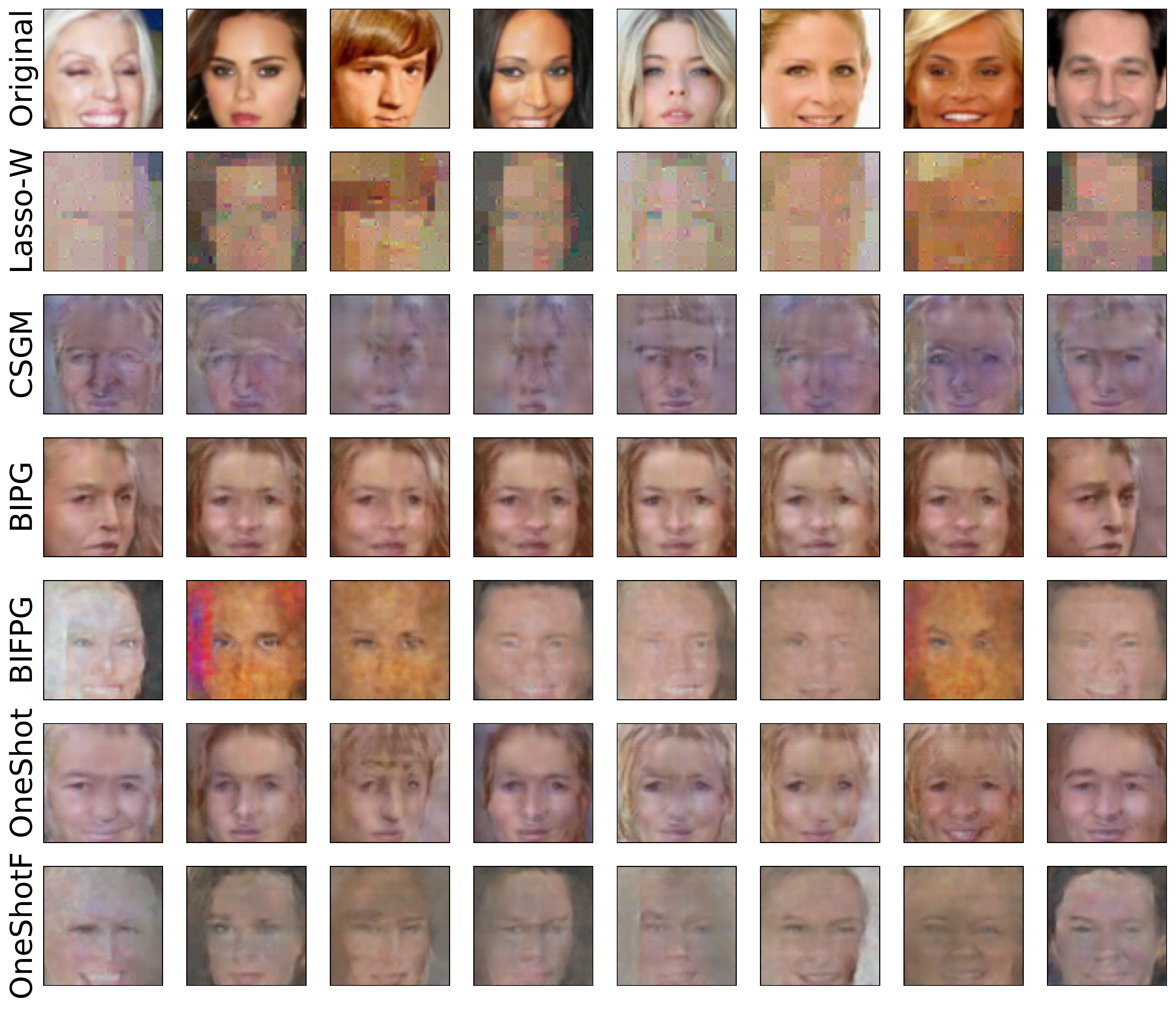} \\
{\small (a)   $\sigma = 0.01 $ and $m = 4000$}\\
\includegraphics[width=0.46\textwidth]{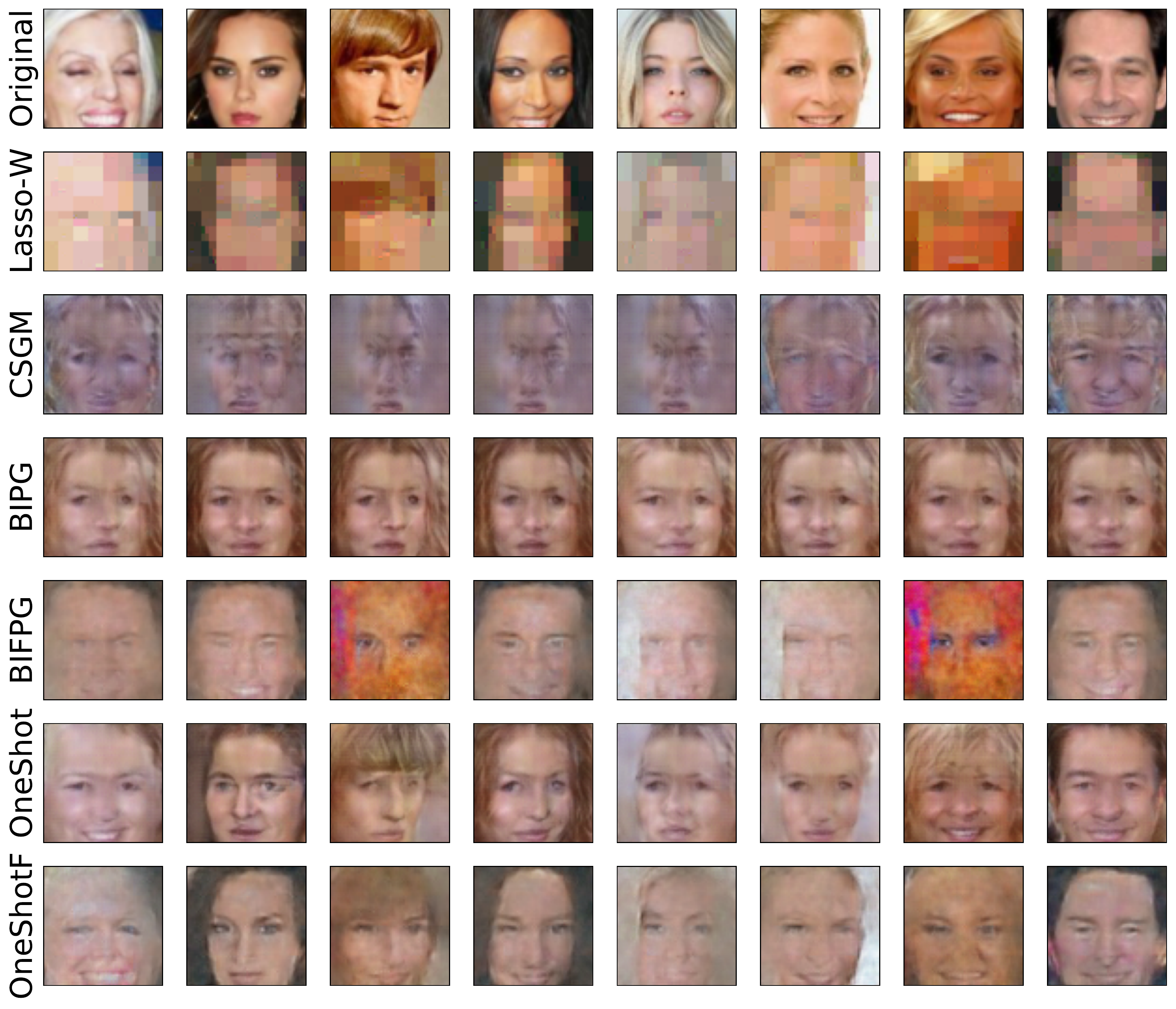}\\
 {\small (b) $\sigma = 0.05 $ and $m = 10000$}\\
\end{tabular}
\caption{Examples of reconstructed images from $1$-bit measurements on CelebA images.}
\label{fig:celebA_1bit}
\end{center}
\end{figure}


\begin{figure}[htp]
\begin{center}
\begin{tabular}{p{3.75cm}<{\centering}p{3.75cm}<{\centering}}
\includegraphics[width=0.23\textwidth]{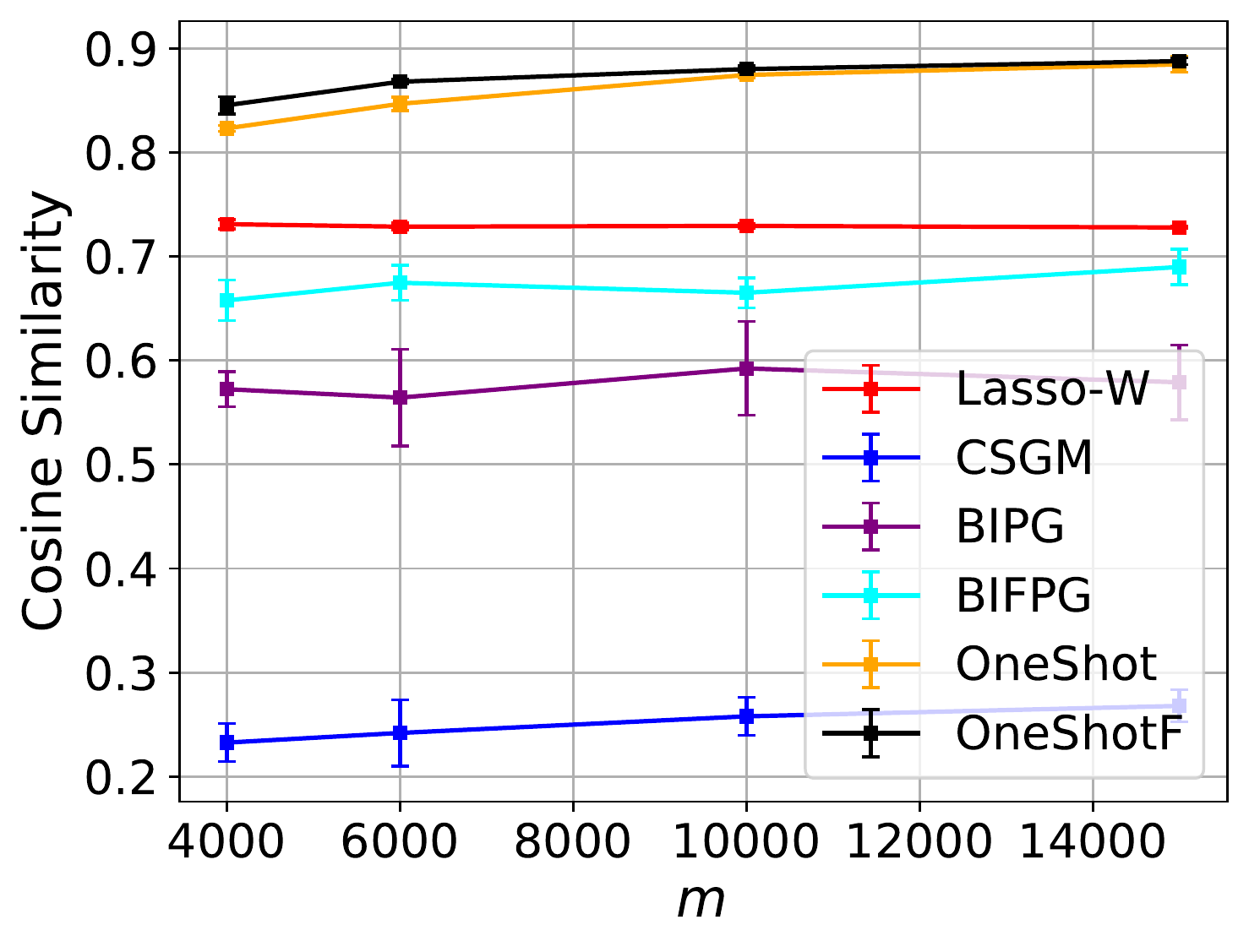} & \includegraphics[width=0.23\textwidth]{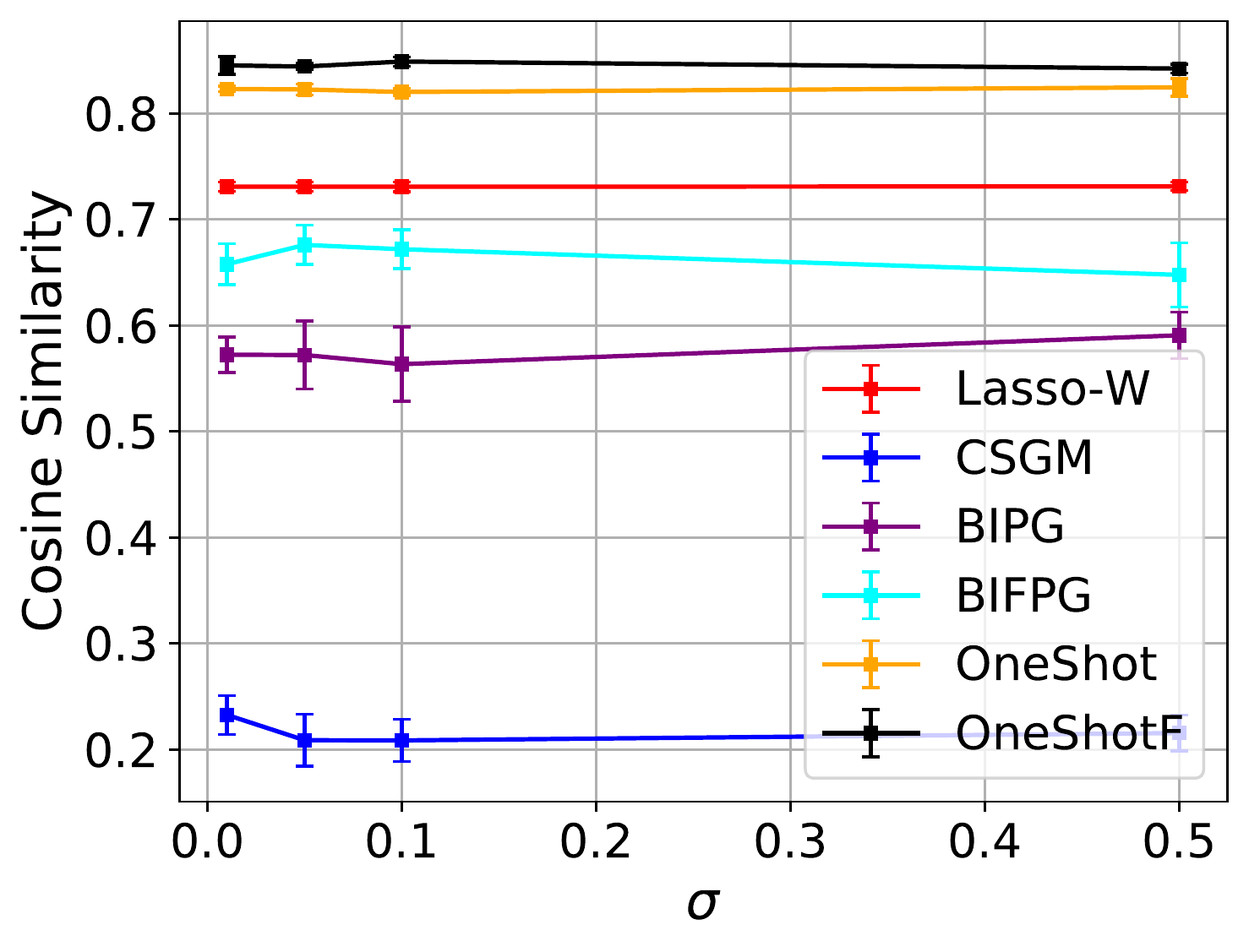}\\
{\small (a) Fixing $\sigma = 0.01 $ and varying $m$  }& {\small (b) Fixing $m = 4000$} and varying $\sigma$ \\
\end{tabular}
\caption{Quantitative comparisons according to the cosine similarity for $1$-bit measurements on CelebA images.}
\label{fig:celebA_1bit_cs}
\end{center}
\end{figure}


\subsection{Recovery Results from Cubic Measurements}
The reconstructed results from cubic measurements for the MNIST dataset are shown in Figure~\ref{fig:mnist_cubic}, and quantitative comparsions in terms of cosine similarity are presented in Figure~\ref{fig:mnist_cubic_cs}. From Figures~\ref{fig:mnist_cubic} and~\ref{fig:mnist_cubic_cs}, we observe that~\texttt{OneShot} outperforms all other competing methods (including~\texttt{OneShotF} and~\texttt{PGD}) by a large margin. The reconstructed results from cubic measurements for the CelebA dataset are shown in Figure~\ref{fig:celebA_cubic}, and quantitative comparsions in terms of cosine similarity are presented in Figure~\ref{fig:celebA_cubic_cs}. From these two figures, we observe that~\texttt{CSGM} almost totally fails to recover the images, and \texttt{OneShot} and \texttt{OneShotF} can still obtain high-quality reconstructed images that are much better than those of \texttt{Lasso-W}. In particular, we observe that~\texttt{OneShot} performs on par with~\texttt{PGD}, for which the first iterative step reduces to~\texttt{OneShot} when setting the initial vector $\bx^{(0)} = \mathbf{0}$ and the step size $\lambda = 1/m$. This reveals that for~\texttt{PGD}, one iterative step may be sufficient, and subsequent iterations might not lead to significant better reconstruction.

\begin{figure}[htp]
\begin{center}
\begin{tabular}{p{3.75cm}<{\centering}p{3.75cm}<{\centering}}
\includegraphics[width=0.23\textwidth]{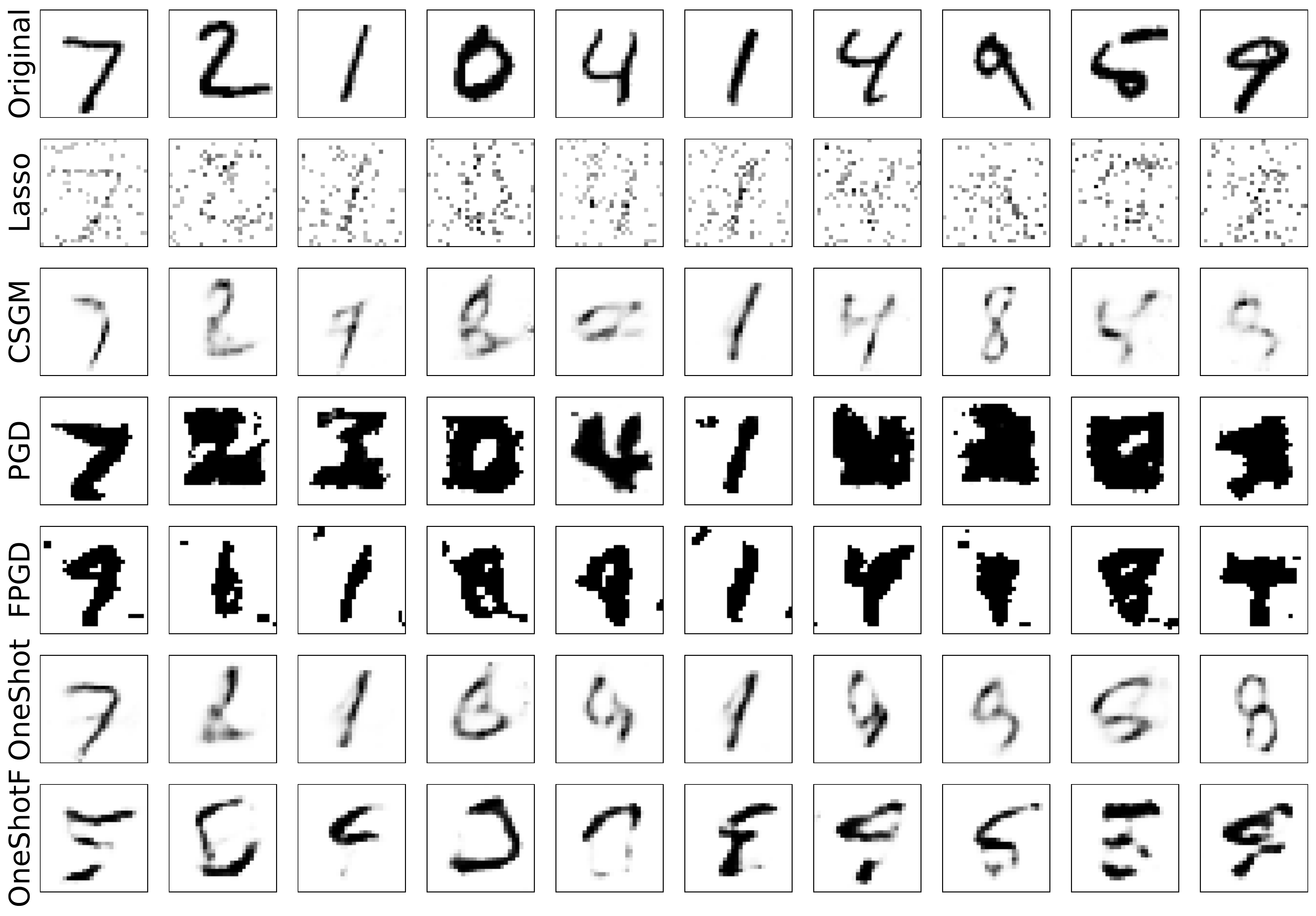} & \includegraphics[width=0.23\textwidth]{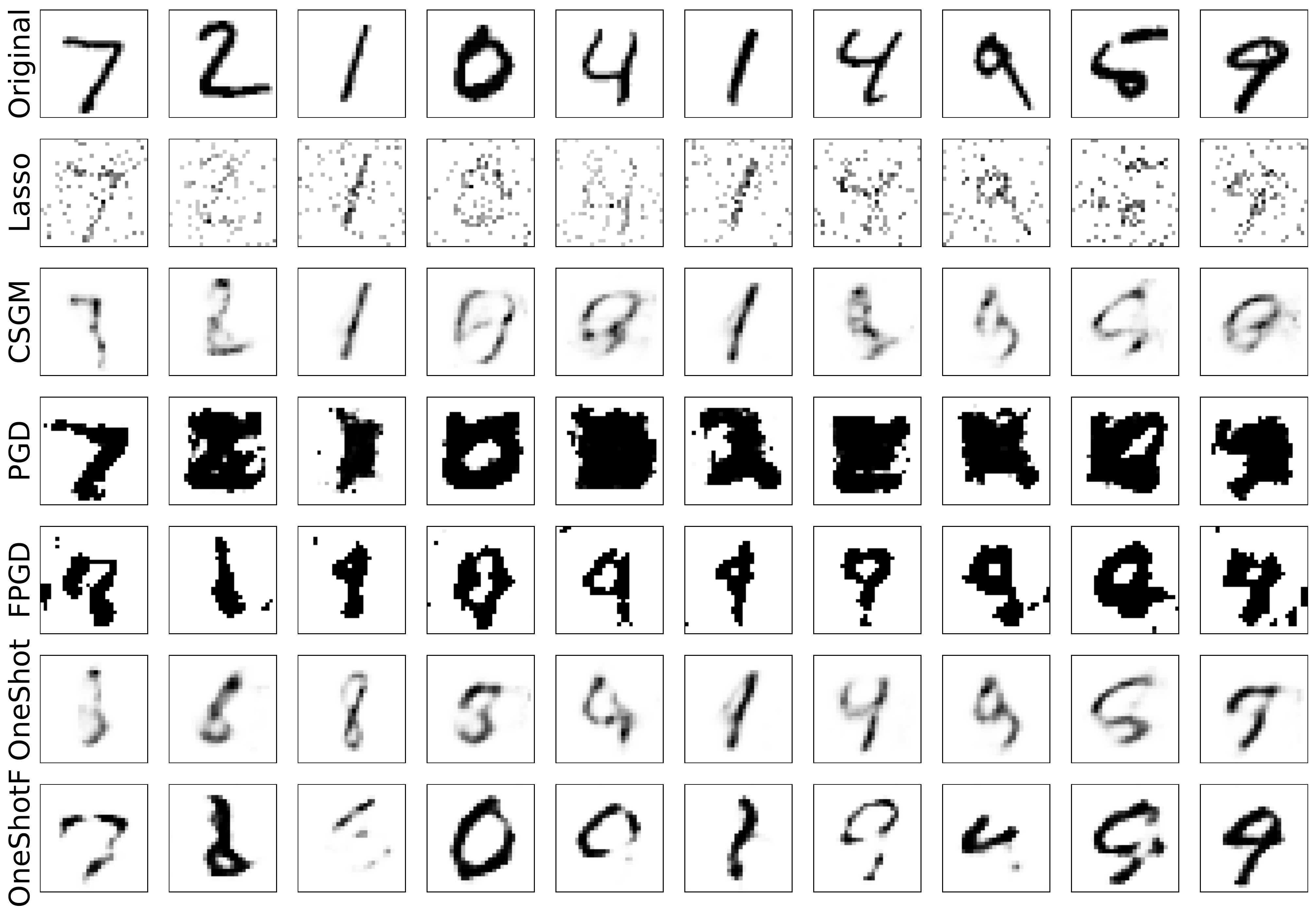}\\
{\small (a) $\sigma = 1$  and  $m = 200$}& {\small (b) $\sigma = 0.1 $  and $m = 400$}\\
\end{tabular}
\caption{Examples of reconstructed images from cubic measurements on MNIST images.}
\label{fig:mnist_cubic}
\end{center}
\end{figure}

\begin{figure}[htp]
\begin{center}
\begin{tabular}{cc}
\includegraphics[height=0.185\textwidth]{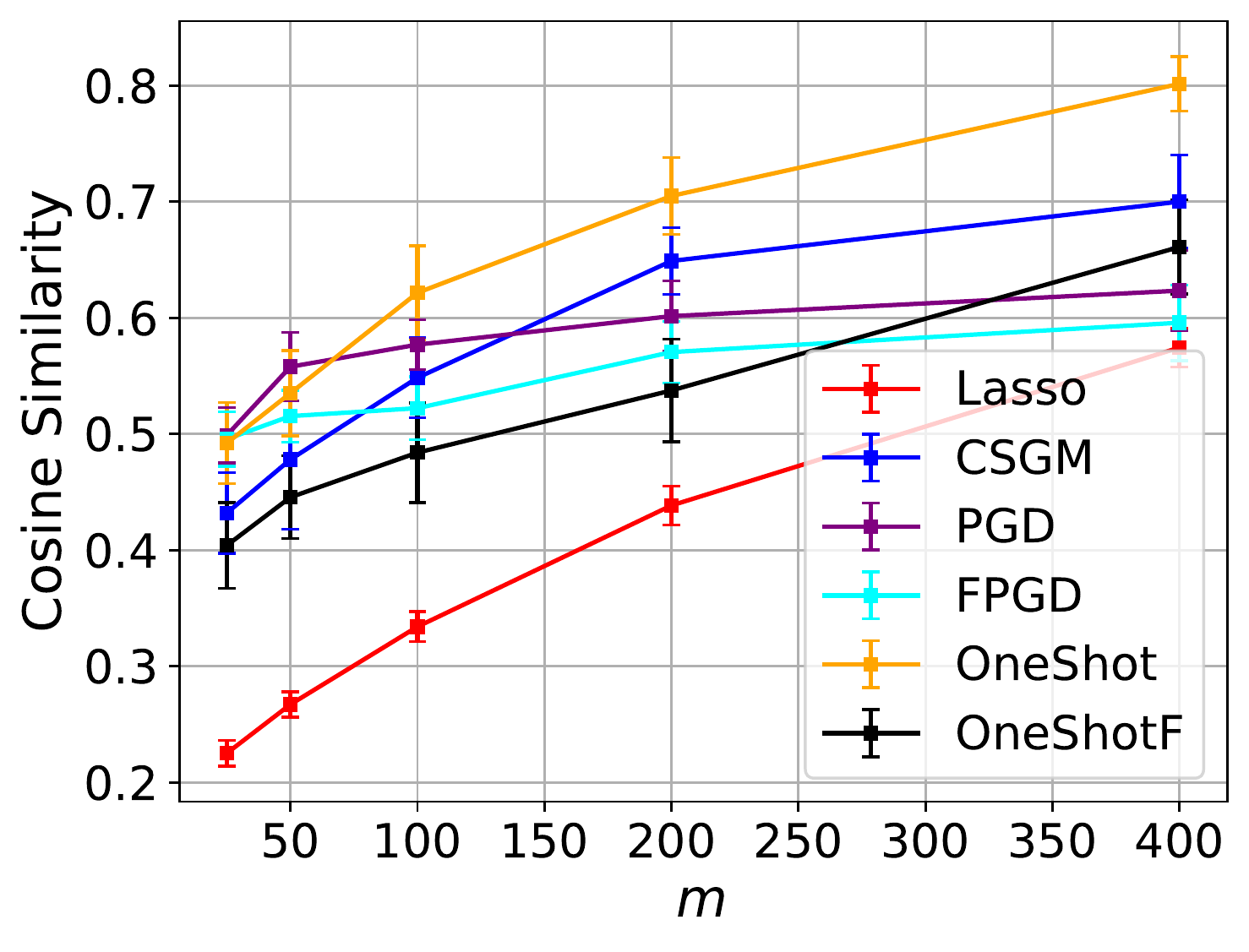} & \hspace{-0.5cm}
\includegraphics[height=0.185\textwidth]{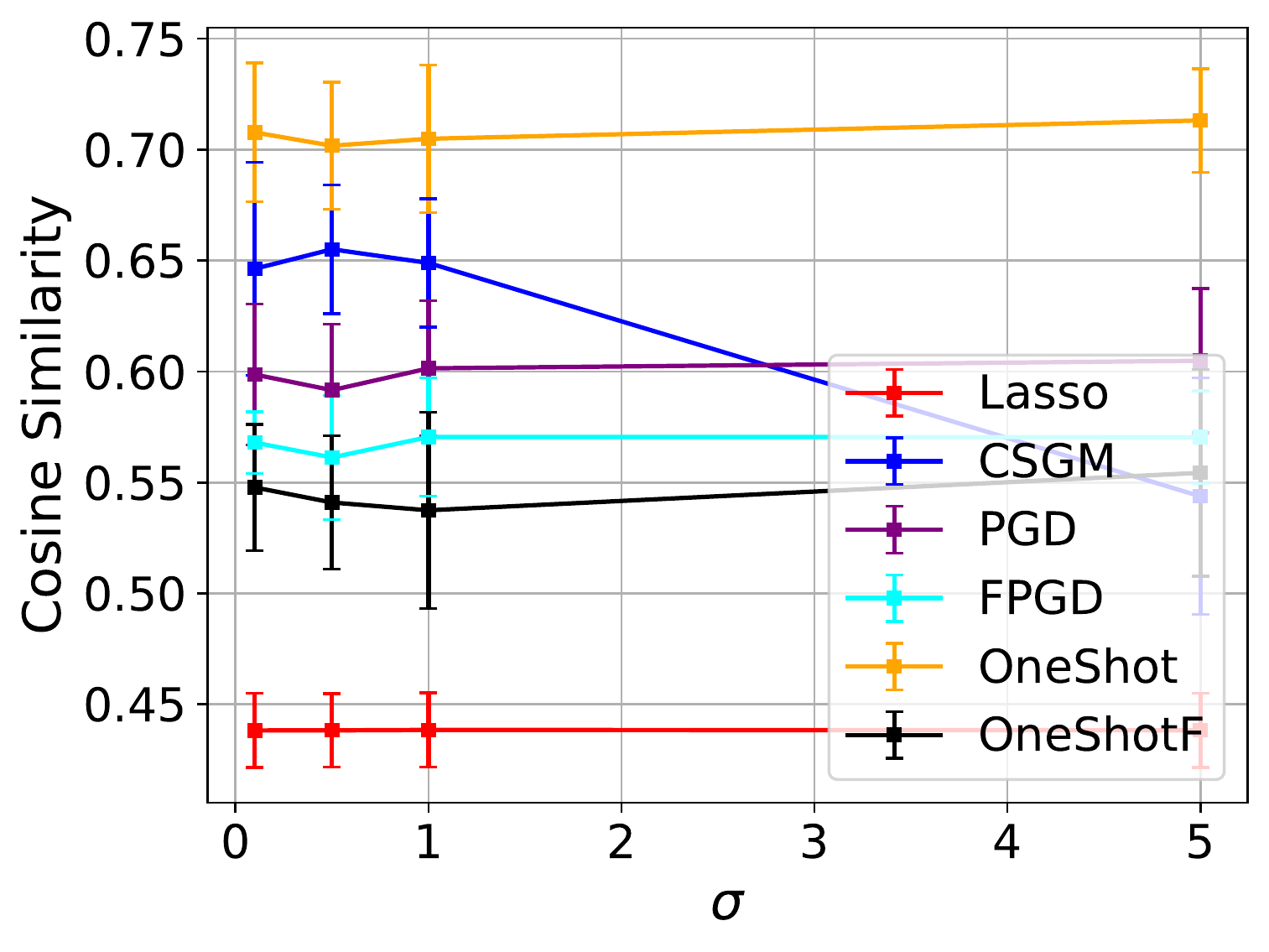} \\
{\small (a) Fixing $\sigma = 1$, varying $m$} & {\small (b) Fixing $m = 200$, varying $\sigma$}
\end{tabular}
\caption{Quantitative comparisons according to the cosine similarity for cubic measurements on MNIST images.}
\label{fig:mnist_cubic_cs}  
\end{center}
\end{figure}

\begin{figure}[htp]
\begin{center}
\begin{tabular}{p{7.5cm}<{\centering}}
\includegraphics[width=0.46\textwidth]{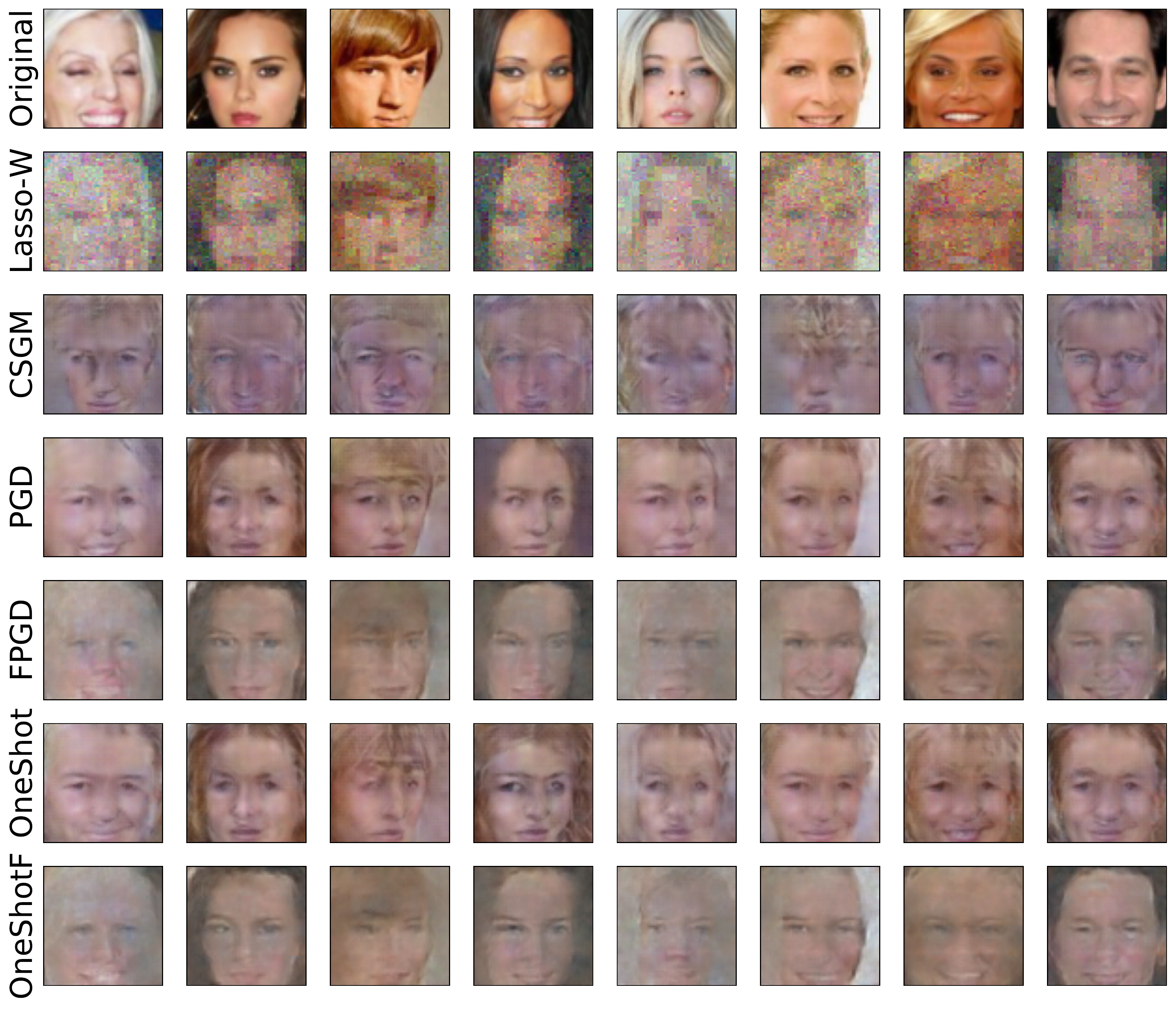} \\
{\small (a)   $\sigma = 0.01 $ and $m = 4000$}\\
\includegraphics[width=0.46\textwidth]{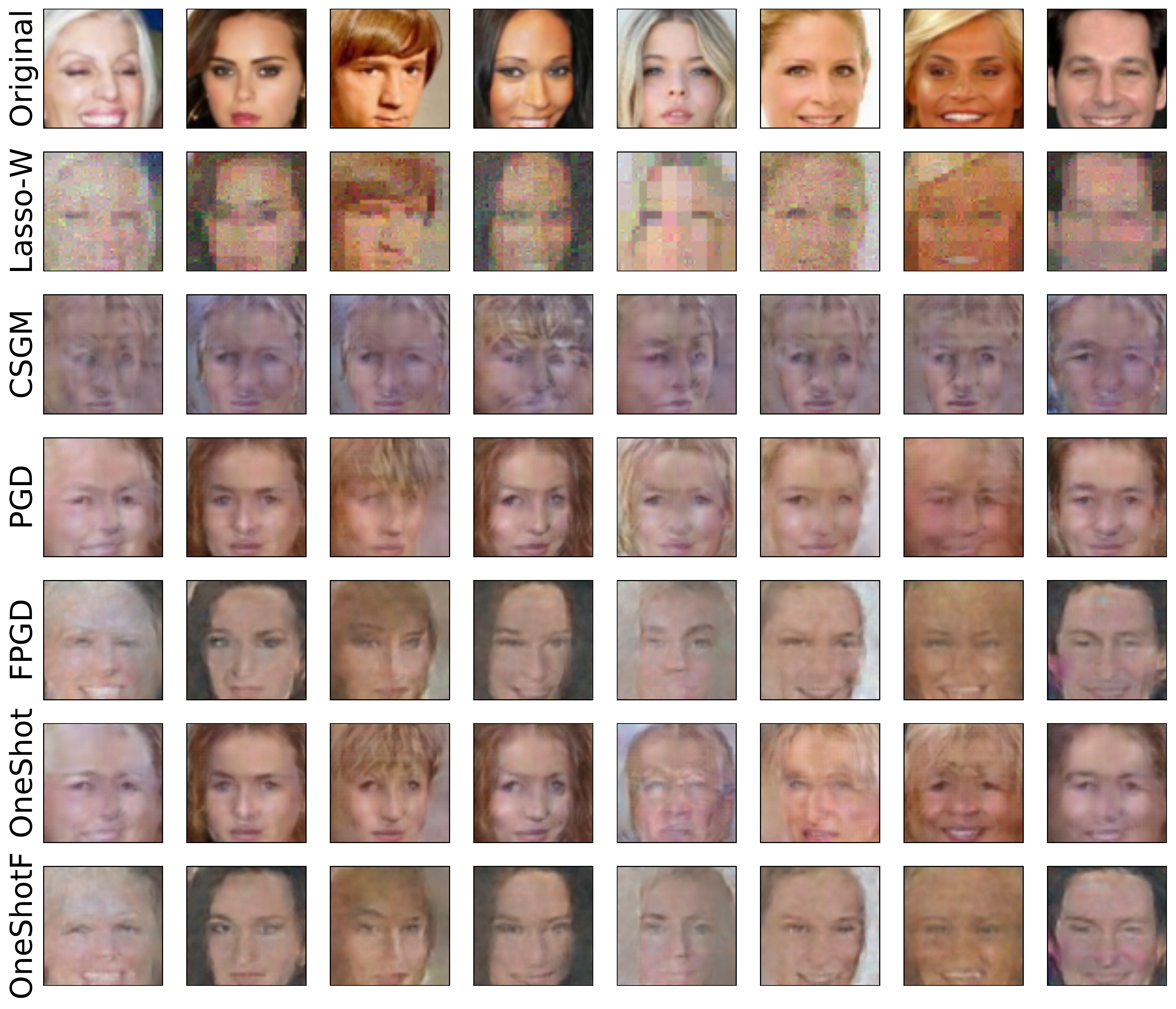}\\
 {\small (b) $\sigma = 0.05 $ and $m = 10000$}\\
\end{tabular}
\caption{Examples of reconstructed images from cubic measurements on CelebA images.}
\label{fig:celebA_cubic}
\end{center}
\end{figure}

\begin{figure}[htp]
\begin{center}
\begin{tabular}{p{3.75cm}<{\centering}p{3.75cm}<{\centering}}
\includegraphics[width=0.23\textwidth]{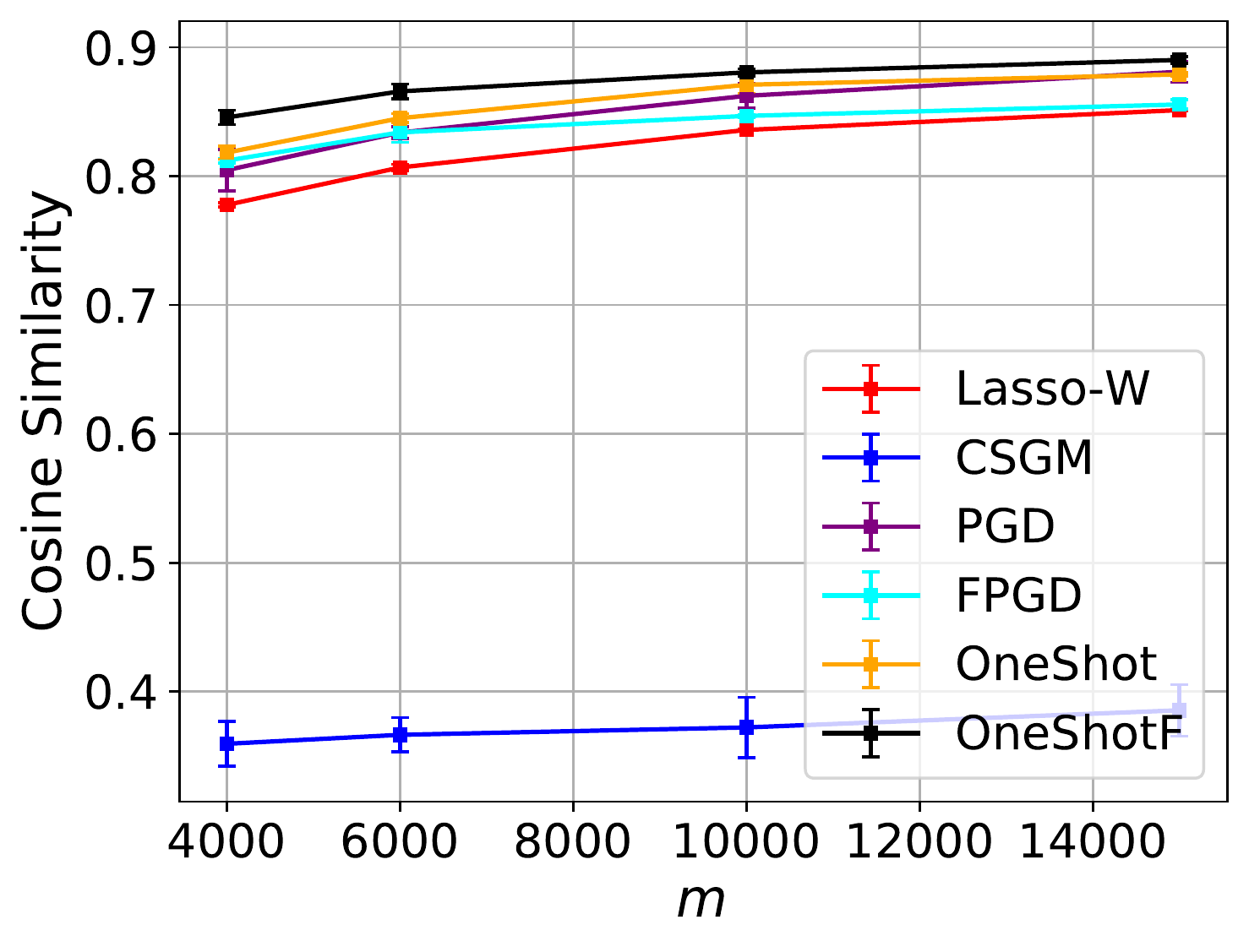} & \includegraphics[width=0.23\textwidth]{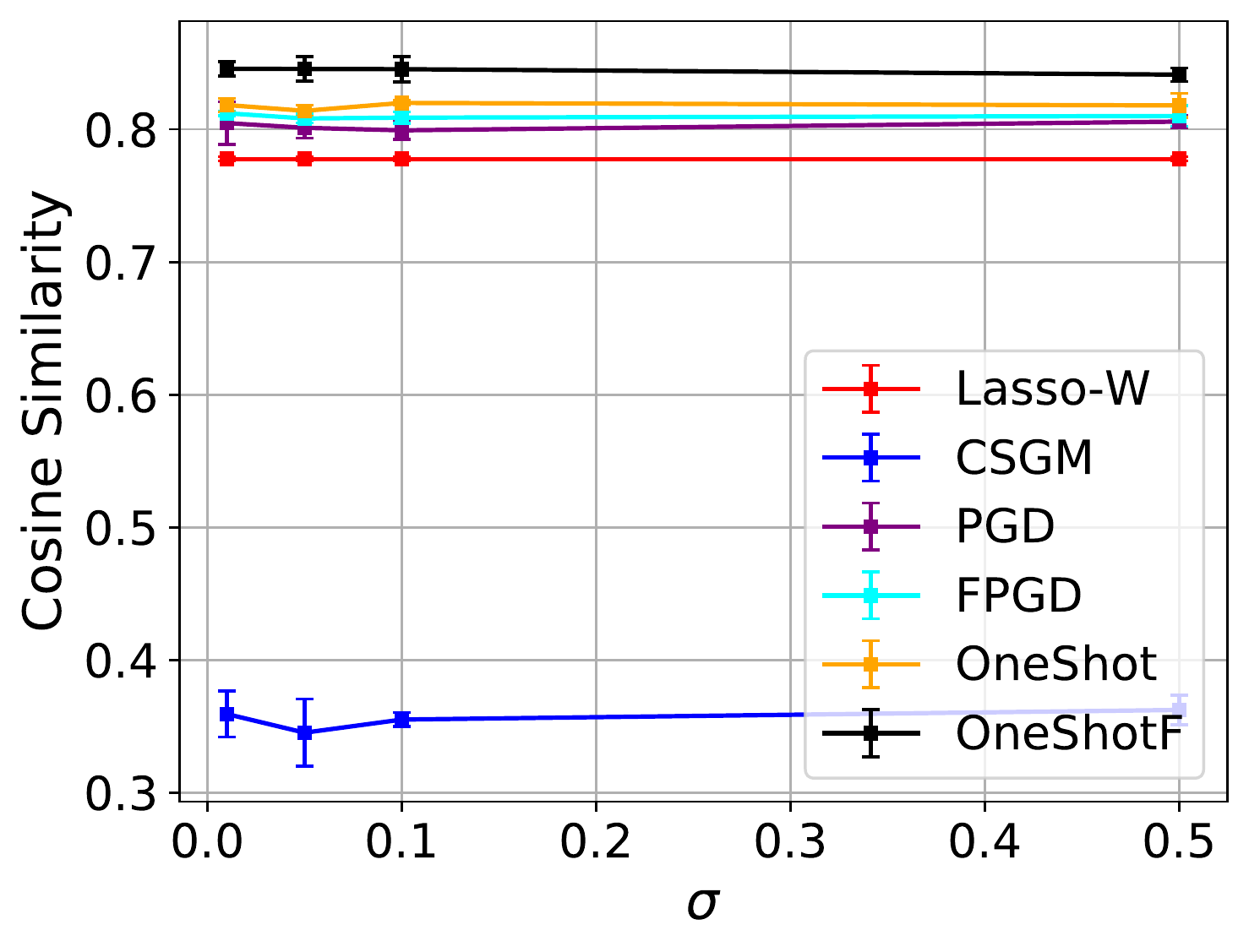}\\
{\small (a) Fixing $\sigma = 0.01 $ and varying $m$  }& {\small (b) Fixing $m = 4000$} and varying $\sigma $ \\
\end{tabular}
\caption{Quantitative comparisons according to the cosine similarity for cubic measurements on CelebA images.}
\label{fig:celebA_cubic_cs}
\end{center}
\end{figure}

\subsection{Running Time}

The running times shown in Table \ref{tab:timeelapsed} illustrate that compared to using gradient-based iterative projection, using GAN-based non-iterative projection leads to much faster computation. \texttt{OneShot} is slower than sparsity-based recovery methods because approximating the projection step is time-consuming. However, since~\texttt{OneShot} only requires one projection, it is much faster than the generative model based projected iterative methods~\texttt{BIPG} and~\texttt{PGD}. Furthermore, recall that from the numerical results, we observe that~\texttt{OneShot} mostly achieves better reconstruction performance compared to that of ~\texttt{BIPG} and~\texttt{PGD}, as well as that of~\texttt{CSGM} and sparsity-based recovery methods.

\begin{table}[htp]
\caption{\label{tab:timeelapsed}The averaged time cost (secs) per reconstruction.} 
\begin{tabular}{|c|>{\raggedleft\arraybackslash}p{1.1cm} >{\raggedleft\arraybackslash}p{1.1cm} |>{\raggedleft\arraybackslash} p{1.1cm}>{\raggedleft\arraybackslash}p{1.1cm}|}
\hline
\multirow{3}{*}{} & \multicolumn{2}{c|}{$1$-bit}                                   & \multicolumn{2}{c|}{cubic}                                   \\ \cline{2-5}
                  & \multicolumn{1}{c|}{MNIST}   & CelebA                        & \multicolumn{1}{c|}{MNIST}   & CelebA                        \\
                  & \multicolumn{1}{c|}{\footnotesize{$m=200$}} & \multicolumn{1}{l|}{\footnotesize{$m=4000$}} & \multicolumn{1}{l|}{\footnotesize{$m=200$}} & \multicolumn{1}{l|}{\footnotesize{$m=4000$}} \\ \hline
\texttt{Lasso}             & \multicolumn{1}{r|}{0.12}    & /                             & \multicolumn{1}{r|}{0.11}    & /                             \\ \hline
\texttt{Lasso-W}           & \multicolumn{1}{r|}{/}       & 12.81                         & \multicolumn{1}{r|}{/}       & 16.27                         \\ \hline
\texttt{CSGM}              & \multicolumn{1}{r|}{0.63}    &  185.26                       & \multicolumn{1}{r|}{0.79}    & 186.54                         \\ \hline
\texttt{BIPG}              & \multicolumn{1}{r|}{14.30}   & 2589.11                             & \multicolumn{1}{r|}{/}   & /                             \\ \hline
\texttt{BIFPG}             & \multicolumn{1}{r|}{0.54}   & 8.93                             & \multicolumn{1}{r|}{/}   & /                             \\ \hline
\texttt{PGD}               & \multicolumn{1}{r|}{/}       &  /                       & \multicolumn{1}{r|}{12.74}       & 2539.24                         \\ \hline
\texttt{FPGD}              & \multicolumn{1}{r|}{/}       &  /                      & \multicolumn{1}{r|}{0.83}       & 8.09                         \\ \hline
\texttt{OneShot}           & \multicolumn{1}{r|}{0.34}    &  130.63                       & \multicolumn{1}{r|}{0.61}    & 128.89                         \\ \hline
\texttt{OneShotF}          & \multicolumn{1}{r|}{0.03}       &  0.61                      & \multicolumn{1}{r|}{0.03}       & 0.66                         \\ \hline
\end{tabular}
\end{table}

\vspace{-1ex}

\section{Conclusion and Future Work}

In this paper, we proposed a non-iterative approach for nonlinear CS with SIMs and generative priors. We made the assumption~\eqref{eq:f4_assumption} for the nonlinear function $f$, and this enabled us to study both the noisy $1$-bit (this cannot be handled by~\cite{wei2019statistical} due to the differentiability assumption) and cubic (this cannot be handled by~\cite{liu2020generalized} due to the sub-Gaussianity assumption for the observations) measurement models. We showed that our approach attained the near-optimal statistical rate $O(\sqrt{(k \log L)/m})$. We also demonstrated via various numerical experiments that our approach was efficient and led to better reconstruction compared to several baselines.

Possible extensions include 1) providing a matching information-theoretic lower bound for CS with SIMs and generative priors; 2) extending to the case that the measurement vectors $\ba_i$ are i.i.d. realizations of $\calN(\mathbf{0},\mathbf{\Sigma})$ with an unknown covariance matrix $\mathbf{\Sigma}$, instead of the current assumption that $\ba_i$ are i.i.d. realizations of $\calN(\mathbf{0},\bI_n)$; 3) training generative models~\cite{gulrajani2017improved,roth2017stabilizing,saharia2021image,karras2021alias} that are more advanced compared to the DCGAN.

\vspace{-1ex}

\section*{Acknowledgments} J. Liu was partially supported by the Fund of the Youth Innovation Promotion Association, CAS (2022002). We gratefully acknowledge Dr. Jonathan Scarlett for proofreading the paper and giving insightful comments.

{\small
\bibliographystyle{ieee_fullname}
\bibliography{writeups}
}

\clearpage
\newpage

\appendix
\onecolumn

{\centering
    {\huge \bf Supplementary Material}

    {\Large \bf Non-Iterative Recovery from Nonlinear Observations using Generative Models (CVPR 2022)

    \large \normalfont Jiulong Liu and Zhaoqiang Liu \par }
}

\section{Overview}

This document presents the supplementary material omitted from the main paper. In Appendix~\ref{app:aux_lemmas}, we provide the proofs of auxiliary lemmas for Theorem~\ref{thm:main}. In Appendix~\ref{app:coro_first}, we provide the proof of Corollary~\ref{coro:first}. We present experimental results with adversarial noise in Appendix~\ref{app:exp_adv}, and we present visualizations of samples generated from the pre-trained generative models in Appendix~\ref{app:exp_vis}. Numbered citations refer to the reference list in the main paper.

\section{Proof of Theorem~\ref{thm:main} (Recovery Guarantee for OneShot)}
\label{app:aux_lemmas}

Before providing the proof, we present some useful auxiliary results.

\subsection{Auxiliary Results for Proving Theorem~\ref{thm:main}}
First, we have a simple tail bound for a Gaussian random variable.
\begin{lemma} {\em (Gaussian tail bound \cite[Example~2.1]{wainwright2019high})}
\label{lem:large_dev_Gaussian} Suppose that $X \sim \calN(\alpha,\sigma^2)$ is a Gaussian random variable with mean $\alpha$ and variance $\sigma^2$. Then, for any $t >0$,
\begin{equation}
\bbP(|X-\alpha|\ge t) \le 2e^{-\frac{t^2}{2\sigma^2}}.
\end{equation}
\end{lemma}

Based on Lemma~\ref{lem:large_dev_Gaussian}, we provide the proof of Lemma~\ref{lem:imp_f4}.

\begin{proof}[Proof of Lemma~\ref{lem:imp_f4}]
Since $\bA \in \bbR^{m\times n}$ is a matrix with i.i.d. standard Gaussian entries, $\bP\bA^T$ is independent of $\bP^{\bot}\bA^T$. Note that the columns of $\bP\bA^T$ are $\langle \ba_i,\bx\rangle \bx$, $y_i = f_i(\langle \ba_i,\bx\rangle)$ with $f_i$ being i.i.d.~realizations of $f$, and the columns of $\bP^{\bot}\bA^T$ are $(\bI_n - \bx\bx^T)\ba_i$. Therefore, for any fixed $\bs$, $\langle (\bI_n - \bx\bx^T)\ba_i, \bs\rangle$ and $y_i$ are mutually independent. In addition,
 \begin{align}
  \frac{1}{m}\sum_{i=1}^m y_i\left\langle \bP^{\bot}\ba_i,\bs\right\rangle = \frac{1}{m}\sum_{i=1}^m y_i\left\langle \left(\bI_n -\bx\bx^T\right)\ba_i,\bs\right\rangle &= \frac{1}{m}\sum_{i=1}^m y_i(\langle\ba_i,\bs\rangle - \langle \bx,\bs\rangle \langle \ba_i,\bx\rangle) \\
  & = \frac{\|\bs\|_2}{m}\sum_{i=1}^m y_i(\langle\ba_i,\bar{\bs}\rangle - \langle \bx,\bar{\bs}\rangle \langle \ba_i,\bx\rangle),\label{eq:lem_imp_eq1}
 \end{align}
 where $\bar{\bs} = \bs/\|\bs\|_2$. Let $g_i = \langle \ba_i,\bx\rangle \sim \calN(0,1)$. Then, since $\mathrm{Cov}[\langle \ba_i,\bar{\bs}\rangle, g_i] = \langle \bx,\bar{\bs}\rangle$, $\langle \ba_i,\bar{\bs}\rangle$ can be written as
 \begin{equation}
  \langle \ba_i,\bar{\bs}\rangle = \langle \bx,\bar{\bs}\rangle g_i + \sqrt{1- \langle \bx,\bar{\bs}\rangle^2} t_i,
 \end{equation}
where $t_i\sim \calN(0,1)$ is independent with $g_i$. Then, we obtain
\begin{align}
 & \frac{\|\bs\|_2}{m}\sum_{i=1}^m y_i(\langle\ba_i,\bar{\bs}\rangle - \langle \bx,\bar{\bs}\rangle \langle \ba_i,\bx\rangle) = \frac{\sqrt{\|\bs\|_2^2 - \langle \bx,\bs\rangle^2}}{m}\sum_{i=1}^m y_i t_i,\label{eq:lem_imp_eq2}
\end{align}
with $t_i$ being standard normal random variables that are independent of $y_i$. Conditioned on the event $\calE$ ({\em cf.}~\eqref{eq:eventE}), combining~\eqref{eq:lem_imp_eq1} and~\eqref{eq:lem_imp_eq2}, we have that $\frac{1}{m}\sum_{i=1}^m y_i\left\langle \bP^{\bot}\ba_i,\bs\right\rangle$ is zero-mean Gaussian with the variance being
\begin{equation}
 \frac{(\|\bs\|_2^2 - \langle \bx,\bs\rangle^2)\sum_{i=1}^m y_i^2}{m^2}.
\end{equation}
From Lemma~\ref{lem:large_dev_Gaussian}, we have for any $u >0$ that
\begin{align}
 \bbP\left(\left|\frac{1}{m}\sum_{i=1}^m y_i\left\langle \bP^{\bot}\ba_i,\bs\right\rangle\right| \ge u\right) &= \exp\left(-\Omega\left(\frac{m u^2}{\left(\|\bs\|_2^2 - \langle \bx,\bs\rangle^2\right)\sum_{i=1}^m y_i^2/m}\right)\right)\\
 & = \exp\left(-\Omega\left(\frac{m u^2}{\|\bs\|_2^2 \xi^2}\right)\right),\label{eq:lem_imp_eq3}
\end{align}
where~\eqref{eq:lem_imp_eq3} follows from $\|\bs\|_2^2 - \langle \bx,\bs\rangle^2 \le \|\bs\|_2^2$ and $\sum_{i=1}^m y_i^2/m \le 2\xi^2$. Setting $\varepsilon = \frac{m u^2}{\|\bs\|_2^2 \xi^2}$, we have $u = \frac{\xi \|\bs\|_2\sqrt{\varepsilon}}{\sqrt{m}}$, and that the desired inequality~\eqref{eq:lem_imp_eq0} holds.
\end{proof}

Next, from Chebyshev's inequality, we have the following simple lemma, which gives~\eqref{eq:cheby_eventE} and an inequality that is useful in the proof of Theorem~\ref{thm:main}.
\begin{lemma}\label{lem:cheby_simple}
 For any $t>0$, with probability at least $1-\frac{\rho^2}{m t^2}$ ({\em cf.}~\eqref{eq:rho_sq}),
 \begin{equation}\label{eq:cheby_simple1}
  \left|\frac{1}{m}\sum_{i=1}^m y_i \langle \ba_i,\bx\rangle -\mu\right| < t.
 \end{equation}
 Similarly, with probability at least $1-\frac{\theta^4}{m \xi^4}$ ({\em cf.}~\eqref{eq:xi_sq} and~\eqref{eq:theta_fourth}),
 \begin{equation}\label{eq:cheby_simple2}
  \frac{1}{m}\sum_{i=1}^m y_i^2 \le 2\xi^2.
 \end{equation}
\end{lemma}
\begin{proof}
 Let $X = \frac{1}{m}\sum_{i=1}^m y_i \langle \ba_i,\bx\rangle -\mu$ and $X_i = y_i \langle \ba_i,\bx\rangle -\mu$. Then, we have $\bbE[X] =0$ and $\mathrm{Var}[X]= \mathrm{Var}\big[\sum_{i=1}^m X_i/m\big] =\mathrm{Var}[X_1]/m = \rho^2/m$. From Chebyshev's inequality, we obtain the desired inequality~\eqref{eq:cheby_simple1}. Similarly, we have $\mathrm{Var}[\sum_{i=1}^m (y_i^2 - \xi^2)/m]=\mathrm{Var}[\sum_{i=1}^m y_i^2/m] = \mathrm{Var}[y_1^2]/m = \theta^4/m$. Again using Chebyshev's inequality, we have
\begin{align}
 \bbP\left(\left|\frac{1}{m}\sum_{i=1}^m y_i^2 - \xi^2\right| > \xi^2 \right) &\le \frac{\mathrm{Var}\left[\sum_{i=1}^m (y_i^2 - \xi^2)/m\right]}{\xi^4} = \frac{\theta^4}{m \xi^4},
\end{align}
which gives~\eqref{eq:cheby_simple2}.
\end{proof}
We are now ready to present the proof of Theorem~\ref{thm:main}.

\subsection{Proof of Theorem~\ref{thm:main}}

Since $\hat{\bx} = \calP_G\big(\frac{1}{m}\bA^T\by\big)$ and $\mu \bx \in \calR(G)$, we have
 \begin{equation}
  \left\|\frac{1}{m}\bA^T\by - \hat{\bx}\right\|_2 \le \left\|\frac{1}{m}\bA^T\by - \mu \bx\right\|_2.
 \end{equation}
Taking square on both sides, we obtain
\begin{equation}
 \left\|\left(\frac{1}{m}\bA^T\by -\mu\bx\right)+(\mu\bx -\hat{\bx})\right\|_2^2 \le \left\|\frac{1}{m}\bA^T\by - \mu \bx\right\|_2^2,
\end{equation}
which leads to
\begin{equation}\label{eq:thm_eq1}
 \|\hat{\bx} - \mu\bx\|_2^2 \le 2\left\langle \frac{1}{m}\bA^T\by -\mu\bx, \hat{\bx} - \mu\bx \right\rangle.
\end{equation}
For $\bP^{\bot} = \bI_n -\bx\bx^T$, we have
\begin{align}
 \frac{1}{m}\bA^T\by -\mu\bx & = \frac{1}{m}\bP^{\bot}\bA^T\by + \frac{1}{m}\bx\bx^T\bA^T\by -\mu\bx \\
 & = \frac{1}{m}\bP^{\bot}\bA^T\by + \left(\frac{1}{m}\bx^T\bA^T\by -\mu\right)\bx.\label{eq:thm_eq2}
\end{align}
Recall that $\calE$ is the event that $\frac{1}{m}\sum_{i=1}^m y_i^2 \le 2\xi^2$, with $\xi^2$ being defined in~\eqref{eq:xi_sq}. For any $u >0$,
\begin{align}
 &\bbP\left(\left|\left\langle\frac{1}{m}\bP^{\bot}\bA^T\by, \hat{\bx} - \mu\bx\right\rangle\right|> u\right) \le \bbP(\calE^c)  +\bbP\left(\left|\left\langle\frac{1}{m}\bP^{\bot}\bA^T\by, \hat{\bx} - \mu\bx\right\rangle\right|> u \Big| \calE\right).
\end{align}
From Lemma~\ref{lem:cheby_simple}, we have $\bbP(\calE^c) \le \frac{\theta^4}{m\xi^4}$, where $\theta^4$ is defined in~\eqref{eq:theta_fourth}. Then, setting $u  = C\left(\xi\sqrt{\frac{k \log \frac{Lr}{\delta}}{m}}\right)(\|\hat{\bx}-\mu\bx\|_2 +\delta)$ with $C>0$ being sufficiently large, from Lemma~\ref{lem:imp_f4} and a chaining argument similar to that in~\cite{bora2017compressed}, we have with probability $1-e^{-\Omega\big(k\log\frac{Lr}{\delta}\big)} - \frac{\theta^4}{m\xi^4}$ that\footnote{For completeness, the proof of~\eqref{eq:thm_eq3} is presented at the end of this section, namely Appendix~\ref{app:proof_chaining}.}
\begin{align}
 \left|\left\langle\frac{1}{m}\bP^{\bot}\bA^T\by, \hat{\bx} - \mu\bx\right\rangle\right| \le u.\label{eq:thm_eq3}
\end{align}
Moreover, we have
\begin{align}
 & \left|\left\langle\left(\frac{1}{m}\bx^T\bA^T\by -\mu\right)\bx, \hat{\bx}-\mu\bx \right\rangle\right| \le \left|\frac{1}{m}\bx^T\bA^T\by -\mu\right| \cdot \|\hat{\bx}-\mu\bx\|_2.
\end{align}
Then, from Lemma~\ref{lem:cheby_simple}, for $\epsilon >0$ and $\rho^2$ defined in~\eqref{eq:rho_sq}, we obtain with probability at least $1-\frac{\rho^2}{m \epsilon^2}$ that
\begin{equation}
 \left|\left\langle\left(\frac{1}{m}\bx^T\bA^T\by -\mu\right)\bx, \hat{\bx}-\mu\bx \right\rangle\right| \le \epsilon  \|\hat{\bx}-\mu\bx\|_2.
\end{equation}
Setting $\epsilon = \xi \sqrt{(k \log \frac{Lr}{\delta})/m}$, we obtain with probability at least $1-\frac{\rho^2}{\xi^2 k \log \frac{Lr}{\delta}}$ that
\begin{align}
 &\left|\left\langle\left(\frac{1}{m}\bx^T\bA^T\by -\mu\right)\bx, \hat{\bx}-\mu\bx \right\rangle\right| \le  \xi \sqrt{\frac{k \log \frac{Lr}{\delta}}{m}}  \cdot \|\hat{\bx}-\mu\bx\|_2.\label{eq:thm_eq4}
\end{align}
Combining~\eqref{eq:thm_eq2},~\eqref{eq:thm_eq3} and~\eqref{eq:thm_eq4}, we obtain with probability $1-e^{-\Omega\big(k\log\frac{Lr}{\delta}\big)} - \frac{\theta^4}{m\xi^4} - \frac{\rho^2}{\xi^2 k \log \frac{Lr}{\delta}}$ that
\begin{align}
 &\left|\left\langle \frac{1}{m}\bA^T\by -\mu\bx, \hat{\bx} - \mu\bx \right\rangle\right| = O\left(\xi\sqrt{\frac{k \log \frac{Lr}{\delta}}{m}}\right)(\|\hat{\bx}-\mu\bx\|_2 +\delta).\label{eq:thm_eq5}
\end{align}
In addition, from~\eqref{eq:thm_eq1}, we have
\begin{equation}
 \|\hat{\bx}-\mu\bx\|_2^2 = O\left(\xi\sqrt{\frac{k \log \frac{Lr}{\delta}}{m}}\right)(\|\hat{\bx}-\mu\bx\|_2 +\delta).
\end{equation}
Then, if
\begin{equation}
 \delta = O\left(\xi\sqrt{\frac{k \log \frac{Lr}{\delta}}{m}}\right),
\end{equation}
we obtain
\begin{equation}\label{eq:thm_eq6}
 \|\hat{\bx}-\mu\bx\|_2 =O\left(\xi\sqrt{\frac{k \log \frac{Lr}{\delta}}{m}}\right),
\end{equation}
which completes the proof.

\subsection{Proof of~\eqref{eq:thm_eq3}}
\label{app:proof_chaining}

Based on Lemma~\ref{lem:imp_f4} and a chaining argument similar to that in~\cite{bora2017compressed}, we have the following lemma that concerns~\eqref{eq:thm_eq3}. 
\begin{lemma}\label{lem:chaining}
 Conditioned on $\calE$, we have that for any $\delta >0$ satisfying $Lr =\Omega(\delta n)$, with probability $1-e^{-\Omega\left(k\log\frac{Lr}{\delta}\right)}$,
 \begin{align}
  & \left|\left\langle \frac{1}{m}\bP^{\bot}\bA^T\by,\hat{\bx}-\mu\bx\right\rangle\right| = O\left(\xi\sqrt{\frac{k \log \frac{Lr}{\delta}}{m}}\right)(\|\hat{\bx}-\mu\bx\|_2 +\delta).
 \end{align}
\end{lemma}
\begin{proof}
  For fixed $\delta >0$ and a positive integer $\ell$, let $M = M_0 \subseteq M_1 \subseteq \ldots \subseteq M_\ell$ be a chain of nets of $B_2^k(r)$ such that $M_i$ is a $\frac{\delta_i}{L}$-net with $\delta_i = \frac{\delta}{2^i}$. There exists such a chain of nets with~\cite[Lemma~5.2]{vershynin2010introduction}
    \begin{equation}
        \log |M_i| \le k \log\frac{4Lr}{\delta_i}. \label{eq:net_size}
    \end{equation}
    By the $L$-Lipschitz continuity of $G$, we have for any $i \in [\ell]$ that $G(M_i)$ is a $\delta_i$-net of $\calR(G)=G(B_2^k(r))$.

    Then, we write
    \begin{align}\label{eq:lem_chain_eq1}
         \hat{\bx} - \mu \bx & = (\hat{\bx} - \hat{\bx}_{\ell}) + \sum_{i=1}^{\ell}(\hat{\bx}_i-\hat{\bx}_{i-1}) + (\hat{\bx}_0 -\mu\bx),
    \end{align}
    where $\hat{\bx}_i \in G(M_i)$ for all $i \in [\ell]$, and $\|\hat{\bx}- \hat{\bx}_{\ell}\| \le \frac{\delta}{2^{\ell}}$, $\|\hat{\bx}_i - \hat{\bx}_{i-1}\|_2 \le \frac{\delta}{2^{i-1}}$ for all $i \in [\ell]$. Therefore, the triangle inequality gives
    \begin{equation}\label{eq:hatbxbx0}
        \|\hat{\bx}-\hat{\bx}_0\|_2 < 2\delta.
    \end{equation}
    Setting $\varepsilon = C k\log\frac{Lr}{\delta}$ with $C>0$ being a sufficiently large constant in Lemma~\ref{lem:imp_f4}, and taking the union bound over $G(M_0)$, we have that with probability $1-e^{-\Omega\big(k \log \frac{Lr}{\delta}\big)}$, for {\em all} $\bs \in G(M_0)$,
    \begin{equation}
     \left|\left\langle \frac{1}{m}\bP^{\bot}\bA^T\by,\bs -\mu\bx\right\rangle\right| = O\left(\frac{\xi \|\bs -\mu\bx\|_2 \sqrt{k \log \frac{Lr}{\delta}}}{\sqrt{m}}\right),
    \end{equation}
    which gives
    \begin{align}
     & \left|\left\langle \frac{1}{m}\bP^{\bot}\bA^T\by,\hat{\bx}_0 -\mu\bx\right\rangle\right|  = O\left(\frac{\xi \|\hat{\bx}_0 -\mu\bx\|_2 \sqrt{k \log \frac{Lr}{\delta}}}{\sqrt{m}}\right).\label{eq:lem_chain_eq2}
    \end{align}
In addition, similarly to that in~\cite{bora2017compressed,liu2020generalized}, we have that if setting $\ell = \lceil \log n \rceil$, when $Lr = \Omega(\delta n)$, with probability $1-e^{-\Omega\big(k \log \frac{Lr}{\delta}\big)}$, it holds that
\begin{equation}\label{eq:lem_chain_eq3}
 \sum_{i=1}^{\ell} \left|\left\langle \frac{1}{m}\bP^{\bot}\bA^T\by,\hat{\bx}_i -\hat{\bx}_{i-1}\right\rangle\right| = O\left(\frac{\xi \delta \sqrt{k \log \frac{Lr}{\delta}}}{\sqrt{m}}\right).
\end{equation}
Moreover, for any $\varepsilon >0$ used in Lemma~\ref{lem:imp_f4}, taking a union bound over $\bs \in \{\be_1,\be_2,\ldots,\be_n\}$, we have with probability $1-ne^{-\Omega(\varepsilon)}$ that
\begin{equation}
 \left\|\frac{1}{m}\bP^{\bot}\bA^T\by\right\|_{\infty} \le \frac{\xi \sqrt{\varepsilon}}{\sqrt{m}}.
\end{equation}
Then, we have
\begin{align}
 \left|\left\langle\frac{1}{m}\bP^{\bot}\bA^T\by, \hat{\bx}-\hat{\bx}_{\ell}\right\rangle\right| & \le \left\|\frac{1}{m}\bP^{\bot}\bA^T\by\right\|_{\infty}\cdot \|\hat{\bx}-\hat{\bx}_{\ell}\|_1 \\
 & \le \frac{\xi \sqrt{\varepsilon}}{\sqrt{m}} \cdot \sqrt{n} \|\hat{\bx}-\hat{\bx}_{\ell}\|_2 \\
 & = O\left(\frac{\xi \delta\sqrt{\varepsilon}}{\sqrt{m}}\right),\label{eq:lem_chain_eq4}
\end{align}
where we use $\|\hat{\bx}-\hat{\bx}_{\ell}\|_2 \le \frac{\delta}{2^{\ell}}$ and the setting $\ell  = \lceil \log n \rceil$ in~\eqref{eq:lem_chain_eq4}. Setting $\varepsilon = C k \log\frac{Lr}{\delta}$ with $C$ being a sufficiently large positive constant in~\eqref{eq:lem_chain_eq4}, we obtain with probability $1-e^{-\Omega\big(k \log \frac{Lr}{\delta}\big)}$ that
\begin{equation}
 \left|\left\langle\frac{1}{m}\bP^{\bot}\bA^T\by, \hat{\bx}-\hat{\bx}_{\ell}\right\rangle\right| = O\left(\frac{\xi \delta \sqrt{k \log \frac{Lr}{\delta}}}{\sqrt{m}}\right).\label{eq:lem_chain_eq5}
\end{equation}
Combining~\eqref{eq:lem_chain_eq1},~\eqref{eq:lem_chain_eq2},~\eqref{eq:lem_chain_eq3}, and~\eqref{eq:lem_chain_eq5}, we obtain the desired result.
\end{proof}

\section{Proof of Corollary~\ref{coro:first} (Extension of Theorem~\ref{thm:main})}
\label{app:coro_first}

Before providing the proof of Corollary~\ref{coro:first}, we present the following useful lemma.
\begin{lemma}{\em (\hspace{1sp}\cite[Lemma~2]{liu2020generalized})}
\label{lem:boraSREC_gen}
    Let $G \,:\, B_2^k(r) \rightarrow \bbR^n$ be $L$-Lipschitz and $\bA \in \bbR^{m\times n}$ be a random matrix with i.i.d. $\calN(0,1)$ entries. For $\alpha >0$ and $\delta >0$, if $m = \Omega\left(\frac{k}{\alpha^2} \log \frac{Lr}{\delta}\right)$, then with probability $1-e^{-\Omega(\alpha^2 m)}$, we have for all $\bx_1,\bx_2 \in G(B_2^k(r))$ that
 \begin{equation}
  \frac{1}{\sqrt{m}}\|\bA \bx_1 - \bA \bx_2\|_2 \le (1+\alpha) \|\bx_1 -\bx_2\|_2 +\delta.
 \end{equation}
\end{lemma}

\begin{proof}[Proof of Corollary~\ref{coro:first}]
 Since $\hat{\bx} = \calP_G\big(\frac{1}{m}\bA^T\by\big)$ and $\tilde{\bx} \in \calR(G)$, we have
 \begin{equation}
  \left\|\frac{1}{m}\bA^T\by - \hat{\bx}\right\|_2 \le \left\|\frac{1}{m}\bA^T\by - \tilde{\bx}\right\|_2.
 \end{equation}
Then, similarly to~\eqref{eq:thm_eq1}, we obtain
\begin{equation}
 \|\hat{\bx}-\tilde{\bx}\|_2^2 \le 2 \left\langle\frac{1}{m}\bA^T\by -\tilde{\bx},\hat{\bx}-\tilde{\bx}\right\rangle.\label{eq:coro1_eq1}
\end{equation}
Let $\tilde{\by} = [f_1(\langle\ba_1,\bx\rangle),\ldots,f_m(\langle\ba_m,\bx\rangle)]^T \in \bbR^m$. We have
\begin{align}
 &\left|\left\langle\frac{1}{m}\bA^T\by -\tilde{\bx},\hat{\bx}-\tilde{\bx}\right\rangle\right| \le \left|\left\langle\frac{1}{m}\bA^T(\by-\tilde{\by}),\hat{\bx}-\tilde{\bx}\right\rangle\right| + \left|\left\langle\frac{1}{m}\bA^T\tilde{\by} -\mu \bx,\hat{\bx}-\tilde{\bx}\right\rangle\right| + \left|\left\langle\mu\bx -\tilde{\bx},\hat{\bx}-\tilde{\bx}\right\rangle\right|.\label{eq:coro1_eq2}
\end{align}
Setting $\alpha = \frac{1}{2}$ in Lemma~\ref{lem:boraSREC_gen}, we obtain that when $m =\Omega(k \log\frac{Lr}{\delta})$, with probability $1-e^{-\Omega(m)}$,
\begin{align}
 \left|\left\langle\frac{1}{m}\bA^T(\by-\tilde{\by}),\hat{\bx}-\tilde{\bx}\right\rangle\right| &= \left|\left\langle\frac{1}{\sqrt{m}}(\by-\tilde{\by}),\frac{1}{\sqrt{m}}\bA(\hat{\bx}-\tilde{\bx})\right\rangle\right|\\
 & \le \left\|\frac{1}{\sqrt{m}}(\by-\tilde{\by})\right\|_2 \cdot \left\|\frac{1}{\sqrt{m}}\bA(\hat{\bx}-\tilde{\bx})\right\|_2\\
 & \le \nu \cdot O\left(\|\hat{\bx}-\tilde{\bx}\|_2 + \delta\right).\label{eq:coro1_eq3}
\end{align}
Similarly to~\eqref{eq:thm_eq5}, we obtain with probability $1-e^{-\Omega\big(k\log\frac{Lr}{\delta}\big)} - \frac{\theta^4}{m\xi^4} - \frac{\rho^2}{\xi^2 k \log \frac{Lr}{\delta}}$ that
\begin{align}
 &\left|\left\langle \frac{1}{m}\bA^T\tilde{\by} -\mu\bx, \hat{\bx} - \tilde{\bx} \right\rangle\right| = O\left(\xi\sqrt{\frac{k \log \frac{Lr}{\delta}}{m}}\right)(\|\hat{\bx}-\tilde{\bx}\|_2 +\delta).\label{eq:coro1_eq4}
\end{align}
Moreover, from the Cauchy–Schwarz inequality, we have
\begin{equation}
 \left|\left\langle\mu\bx -\tilde{\bx},\hat{\bx}-\tilde{\bx}\right\rangle\right| \le \|\mu\bx -\tilde{\bx}\|_2 \cdot \|\hat{\bx}-\tilde{\bx}\|_2.\label{eq:coro1_eq5}
\end{equation}
Combining~\eqref{eq:coro1_eq1},~\eqref{eq:coro1_eq2},~\eqref{eq:coro1_eq3},~\eqref{eq:coro1_eq4} and~\eqref{eq:coro1_eq5}, we obtain that when $m =\Omega\big(k\log\frac{Lr}{\delta}\big)$, with probability $1-e^{-\Omega\big(k\log\frac{Lr}{\delta}\big)} - \frac{\theta^4}{m\xi^4} - \frac{\rho^2}{\xi^2 k \log \frac{Lr}{\delta}}$,
\begin{align}
 &\|\hat{\bx}-\tilde{\bx}\|_2^2 \le \left(\xi\sqrt{\frac{k \log \frac{Lr}{\delta}}{m}} + \nu + \|\mu\bx -\tilde{\bx}\|_2\right) (\|\hat{\bx}-\tilde{\bx}\|_2 + \delta).
\end{align}
From the triangle inequality $\|\hat{\bx}-\mu\bx\|_2 \le \|\hat{\bx}-\tilde{\bx}\|_2 + \|\tilde{\bx}-\mu\bx\|_2$, and similarly to~\eqref{eq:thm_eq6}, we obtain the desired result.
\end{proof}

\section{Supplementary Experimental Results with Adversarial Noise}\label{app:exp_adv}
In this section, we present the supplementary numerical results for the SIM ({\em cf.}~\eqref{eq:sim}) with adversarial noise, for a noisy $1$-bit measurement model
\begin{equation}
\label{eqn:1bitm_imgn}
 y_i = \mbox{sign}(\langle \ba_i,\bx+e_i\rangle), \quad i \in [m],
\end{equation}
where $e_{i}$ are i.i.d.~realizations of $\mathcal{N}\big(0, \sigma^{2}\big)$, and a noisy cubic measurement model
\begin{equation}
\label{eqn:cubicm_imgn}
 y_i =\langle \ba_i,\bx+ \eta_i\rangle^{3}, \quad i \in [m],
\end{equation}
where $\eta_{i}$ are i.i.d.~realizations $\mathcal{N}\big(0, \sigma^{2}\big)$.

The reconstructed results from $1$-bit and cubic measurements are shown in Figures~\ref{fig:mnist_1bit_imgn} and~\ref{fig:mnist_cubic_imgn} respectively. We can observe that~\texttt{OneShot}  outperforms~\texttt{Lasso}, \texttt{CSGM},~\texttt{BIPG} and~\texttt{BIFPG} (or \texttt{PGD} and~ \texttt{FPGD}) by a large margin and it also leads to superior performance over \texttt{OneShotF}. In addition, the cosine similarities plotted in Figures~\ref{fig:mnist_1bit_cs_imgn} and~\ref{fig:mnist_cubic_cs_imgn} illustrate that our method \texttt{OneShot} mostly outperforms all other competing methods.

  \begin{figure} 
\begin{center}
\begin{tabular}{cc}
\includegraphics[height=0.3\textwidth]{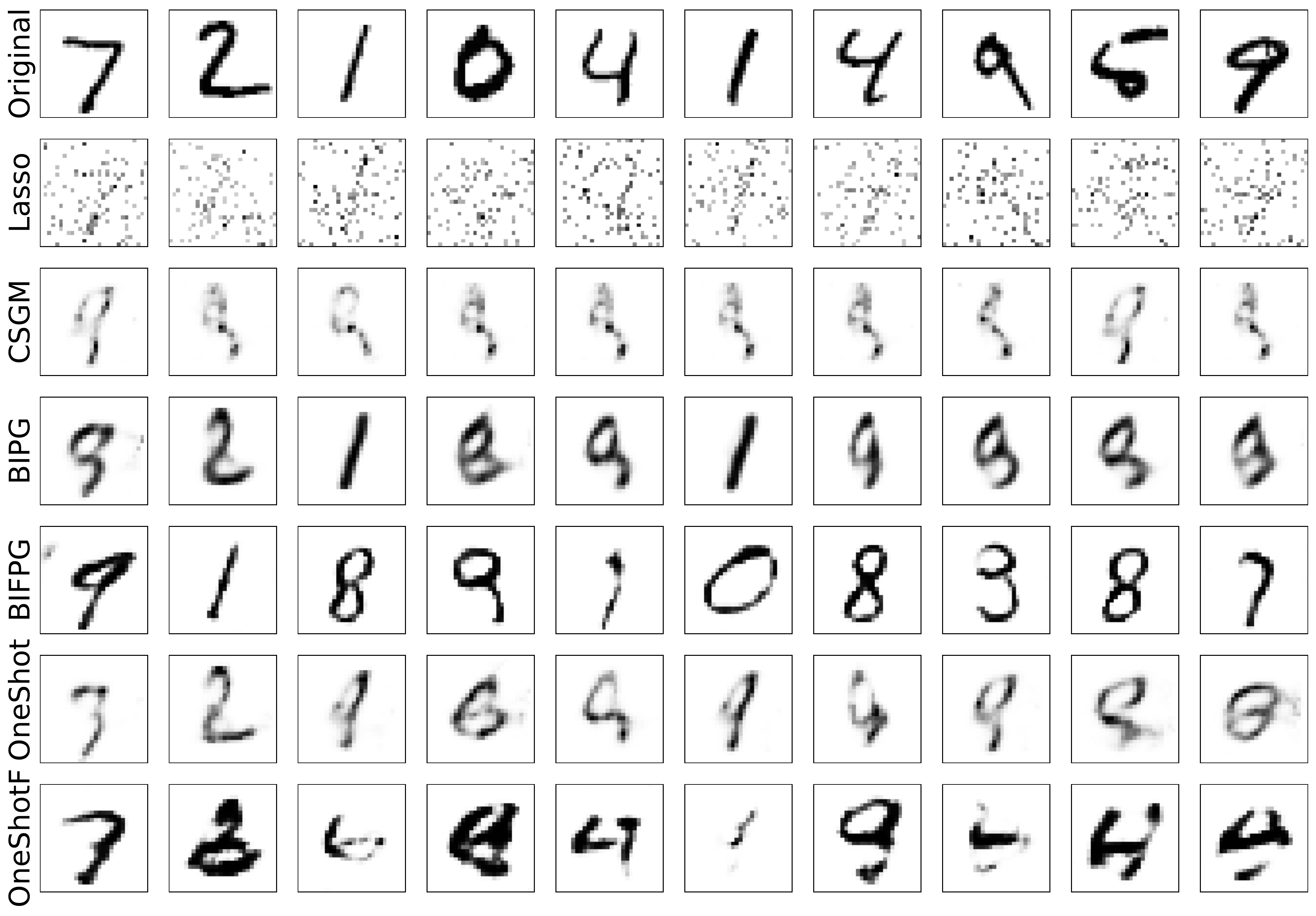} & \hspace{0.5cm}
\includegraphics[height=0.3\textwidth]{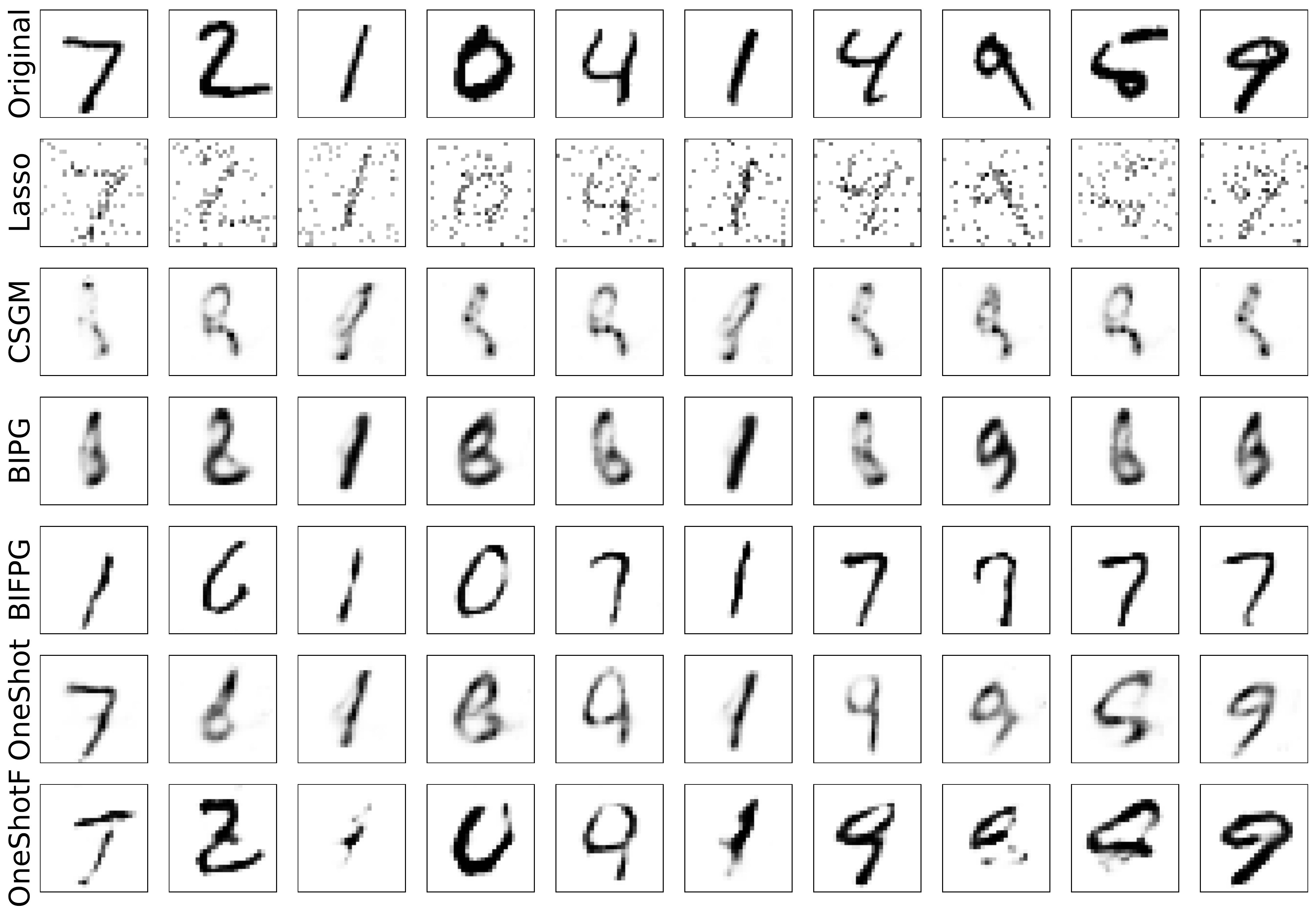} \\
{\small (a) $\sigma = 0.1 $   and $m = 200$} & {\small (b) $\sigma = 0.01 $  and $m = 400$}
\end{tabular}
\caption{Examples of reconstructed images from adversarially corrupted $1$-bit measurements on MNIST images.}
\label{fig:mnist_1bit_imgn}
\end{center}
\end{figure} 

  \begin{figure} 
\begin{center}
\begin{tabular}{cc}
\includegraphics[height=0.3\textwidth]{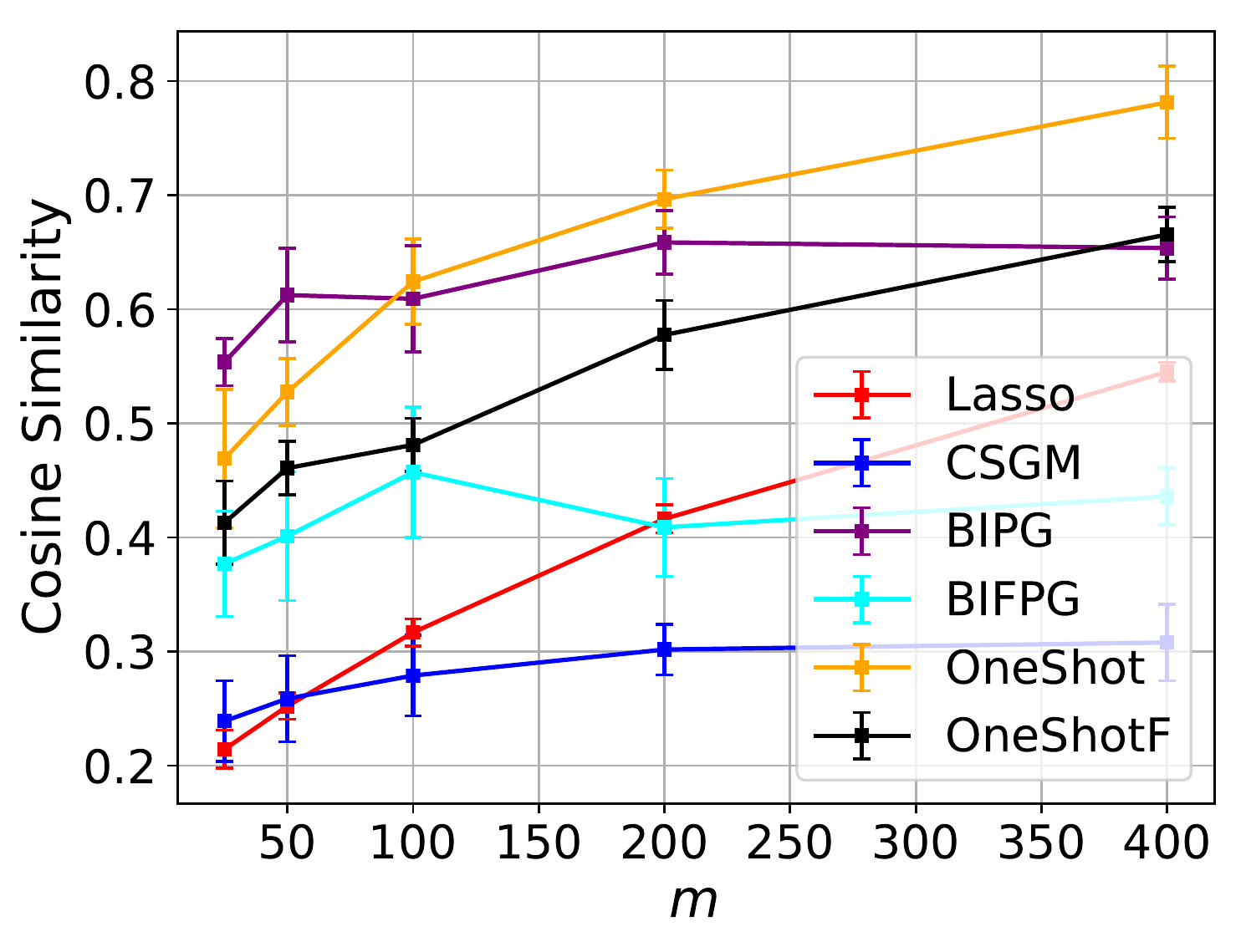} & \hspace{0.5cm}
\includegraphics[height=0.3\textwidth]{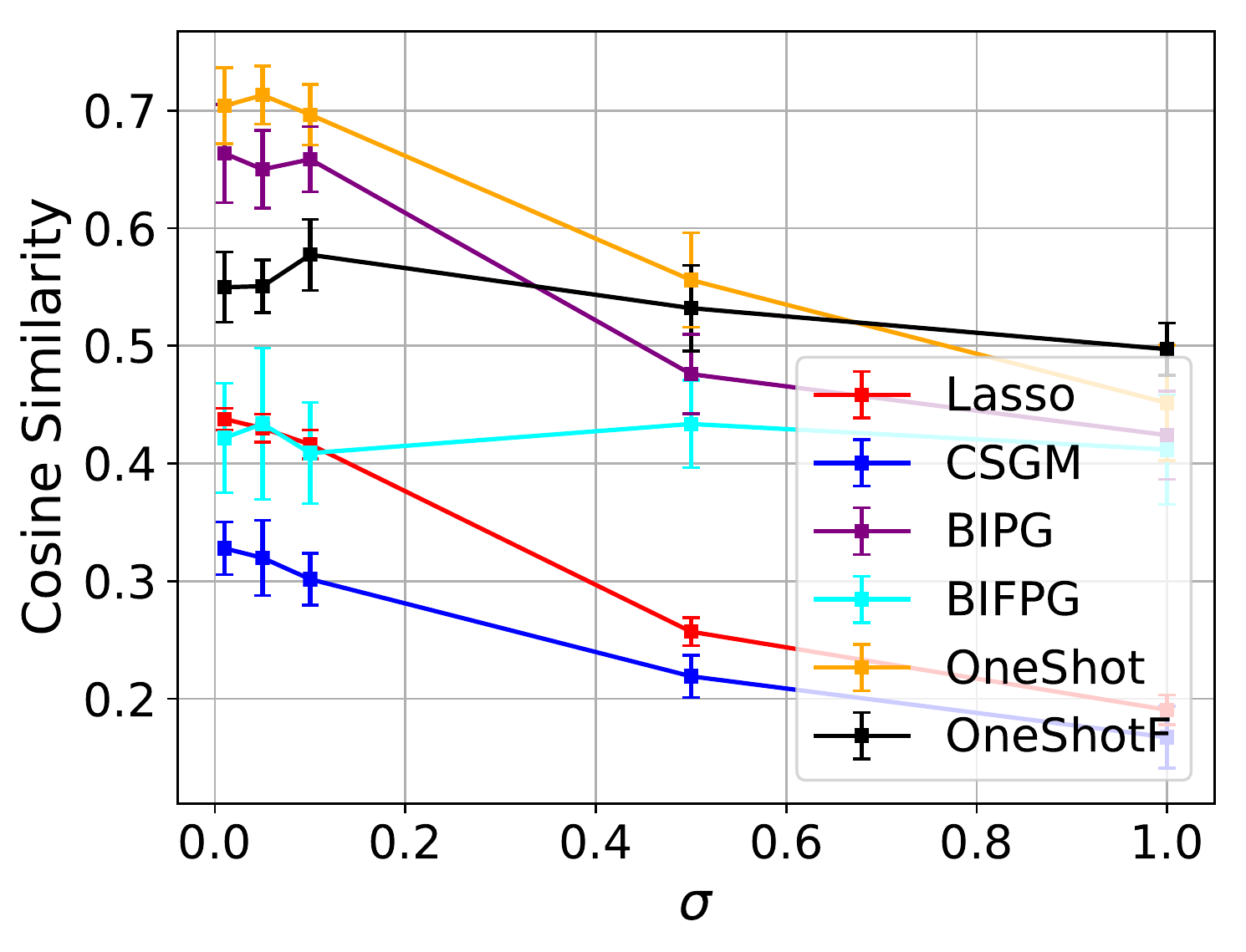} \\
{\small (a) Fixing $\sigma = 0.1$  and varying $m$} & {\small (b) Fixing $m = 200$ and varying $\sigma$}
\end{tabular}
\caption{Quantitative comparisons according to the cosine similarity for adversarially corrupted $1$-bit measurements on MNIST images.}
\label{fig:mnist_1bit_cs_imgn}
\end{center}
\end{figure} 

  \begin{figure} 
\begin{center}
\begin{tabular}{cc}
\includegraphics[height=0.3\textwidth]{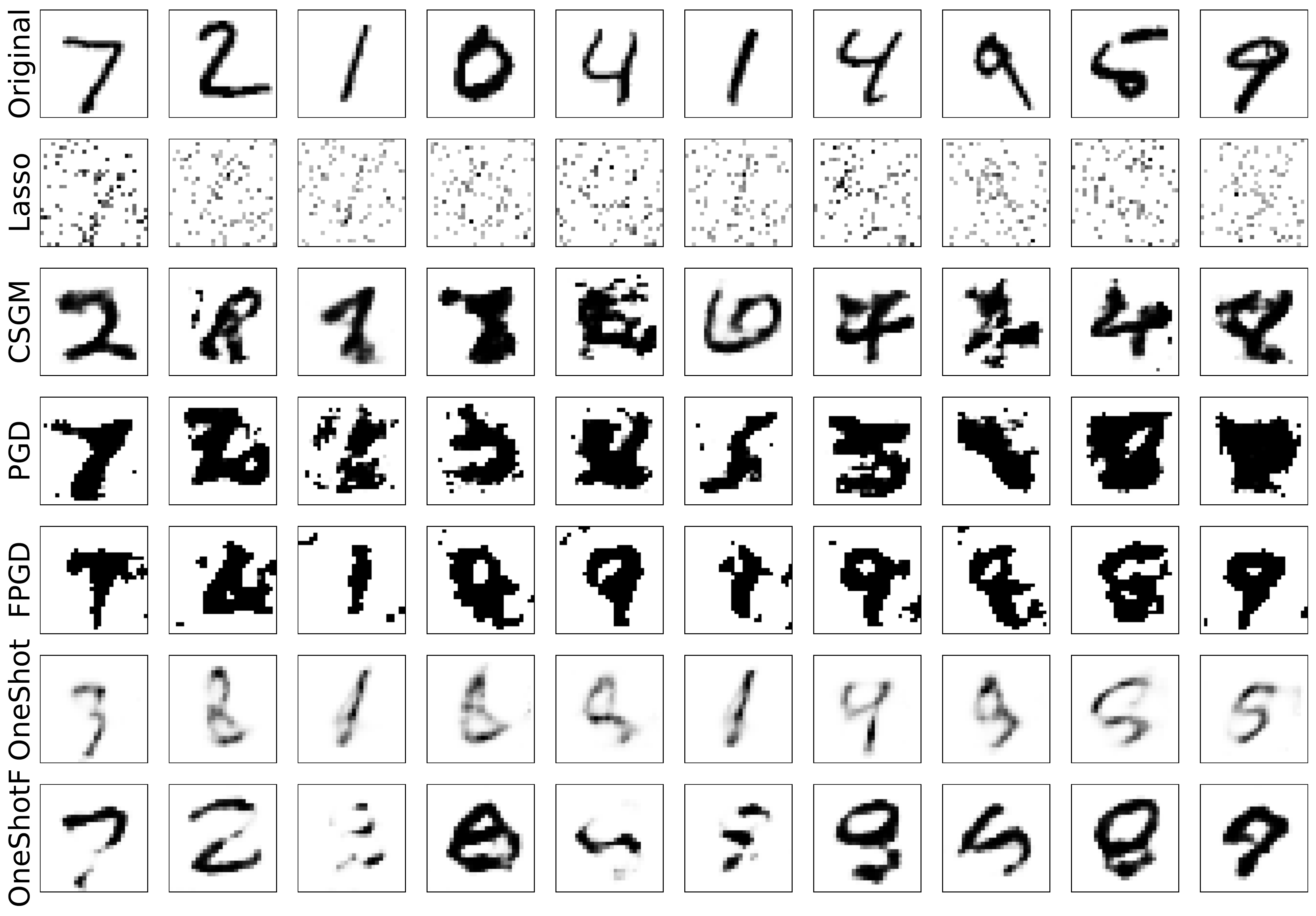} & \hspace{0.5cm}
\includegraphics[height=0.3\textwidth]{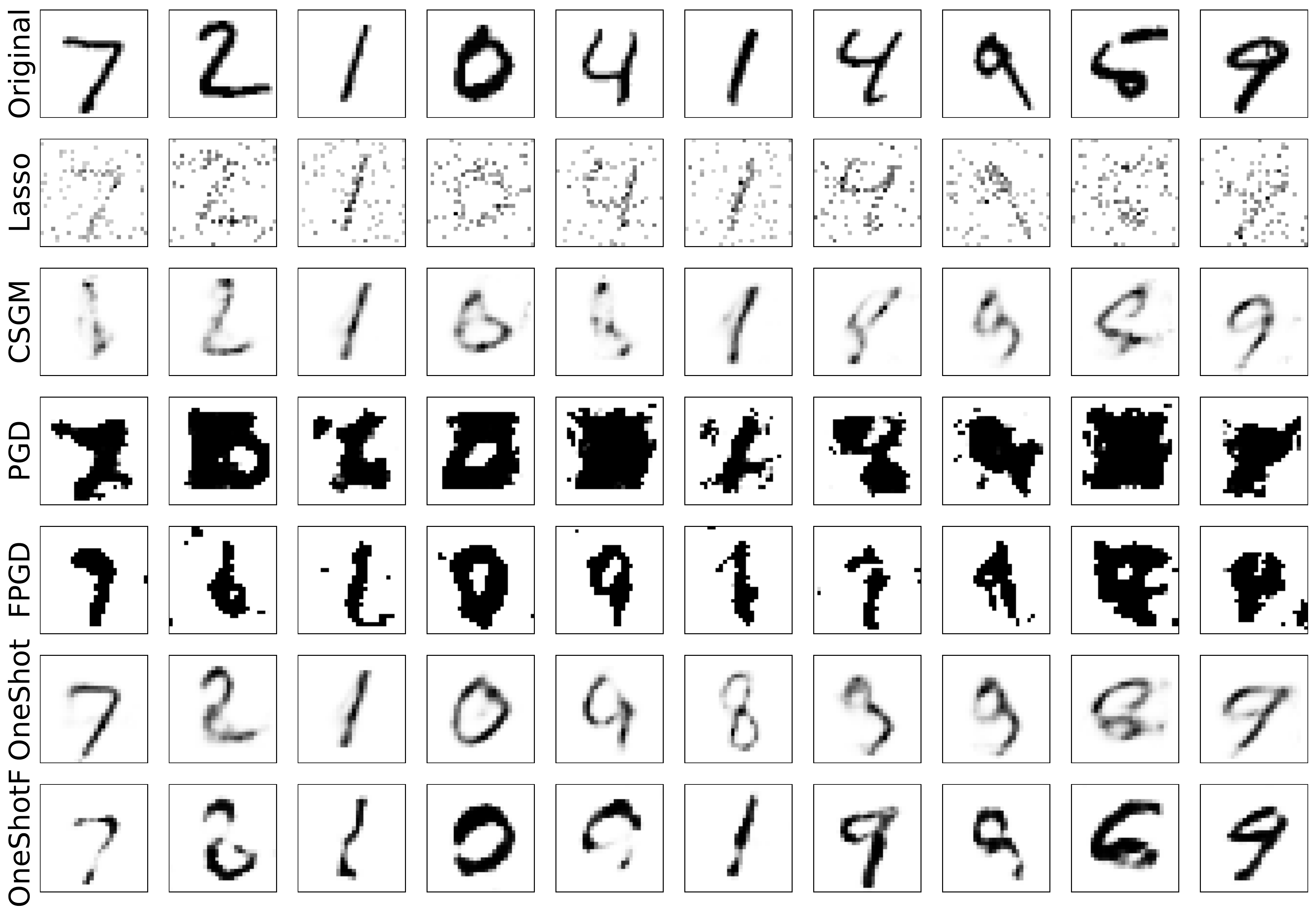} \\
{\small (a) $\sigma = 0.1 $  and  $m = 200$} & {\small (b) $\sigma = 0.01 $  and $m = 400$}
\end{tabular}
\caption{Examples of reconstructed images from adversarially corrupted cubic measurements on MNIST images.}
\label{fig:mnist_cubic_imgn}
\end{center}
\end{figure} 

  \begin{figure} 
\begin{center}
\begin{tabular}{cc}
\includegraphics[height=0.3\textwidth]{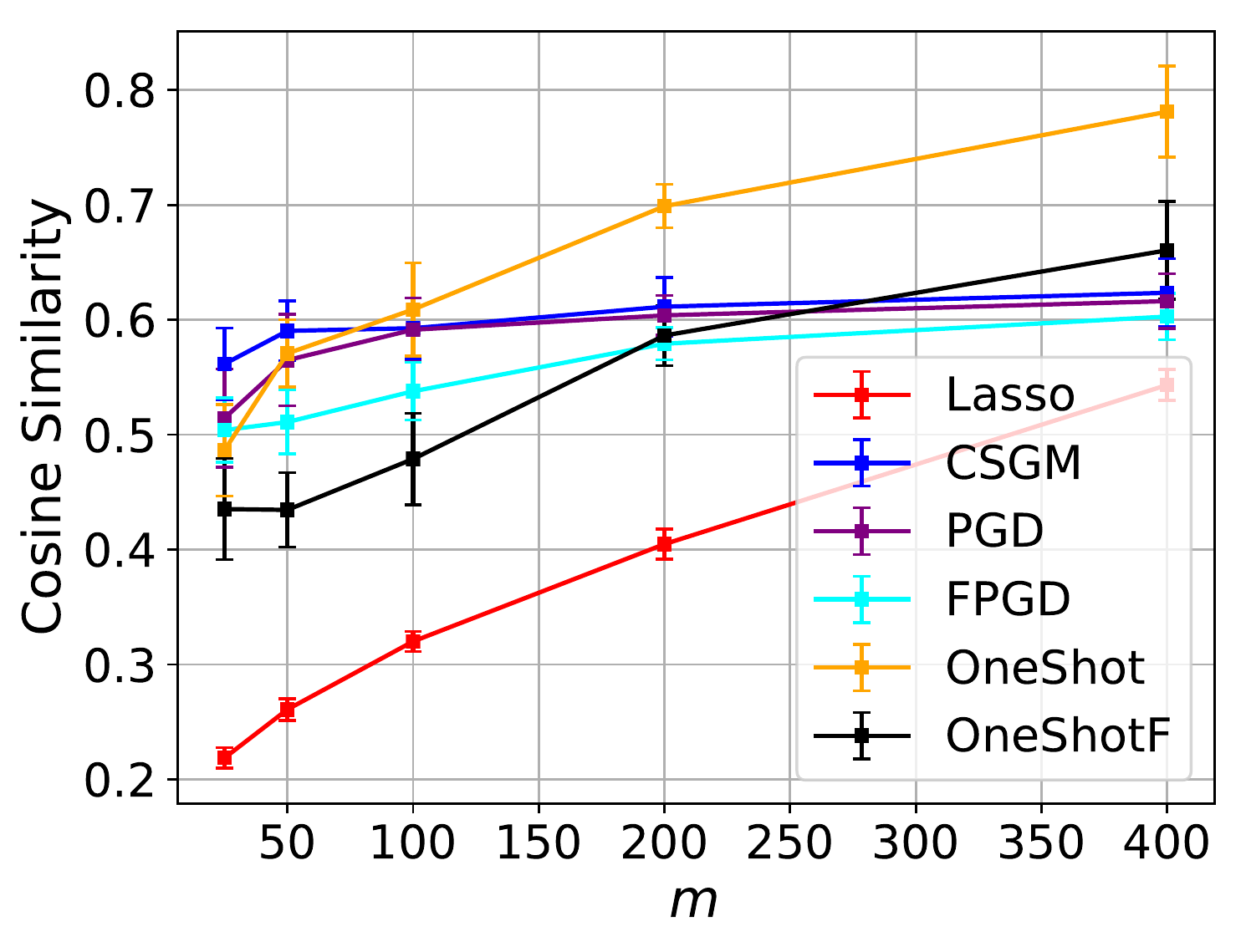} & \hspace{0.5cm}
\includegraphics[height=0.3\textwidth]{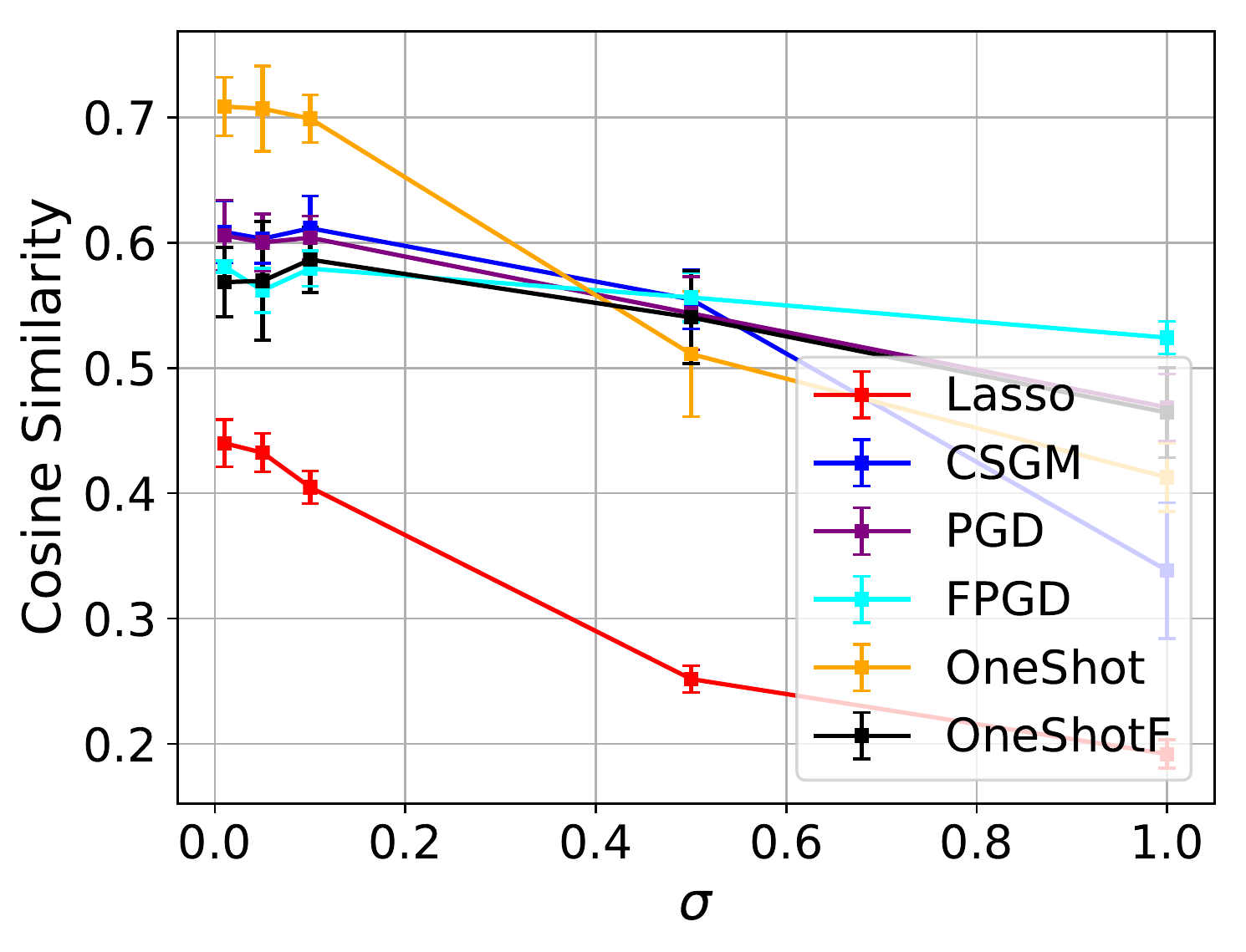} \\
{\small (a) Fixing $\sigma = 0.1$, varying $m$} & {\small (b) Fixing $m = 200$, varying $\sigma$}
\end{tabular}
\caption{Quantitative comparisons according to the cosine similarity for adversarially corrupted cubic measurements on MNIST images.}
\label{fig:mnist_cubic_cs_imgn}
\end{center}
\end{figure} 

\section{Visualization of Samples Generated from the Pre-trained Generative Models}\label{app:exp_vis}
The samples generated from the pre-trained VAE and the pre-trained DCGAN that are used for the gradient-based projection method in this paper are shown in Figure~\ref{fig:samples}. Though some samples generated from the two classic generative models are not perfect and they are distinguishable from images in the datasets, our proposed method still achieves the SOTA performance. We may investigate our method using the SOTA generative models such as those in~\cite{gulrajani2017improved,roth2017stabilizing,saharia2021image,karras2021alias} in future works to further improve the performance.

\begin{figure} 
\begin{center}
\begin{tabular}{cc}
\includegraphics[height=0.4\textwidth]{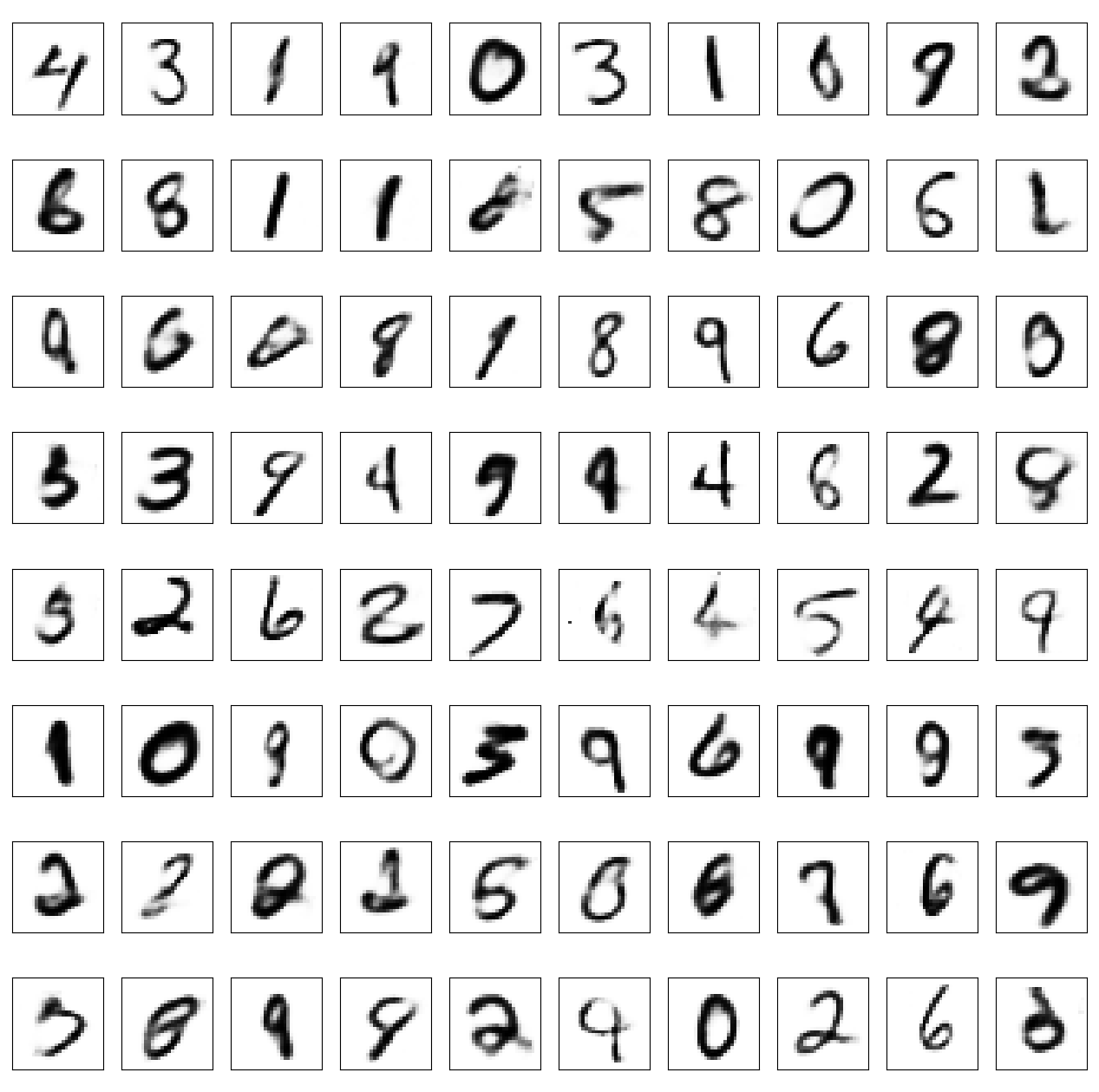} & \hspace{0.5cm}
\includegraphics[height=0.4\textwidth]{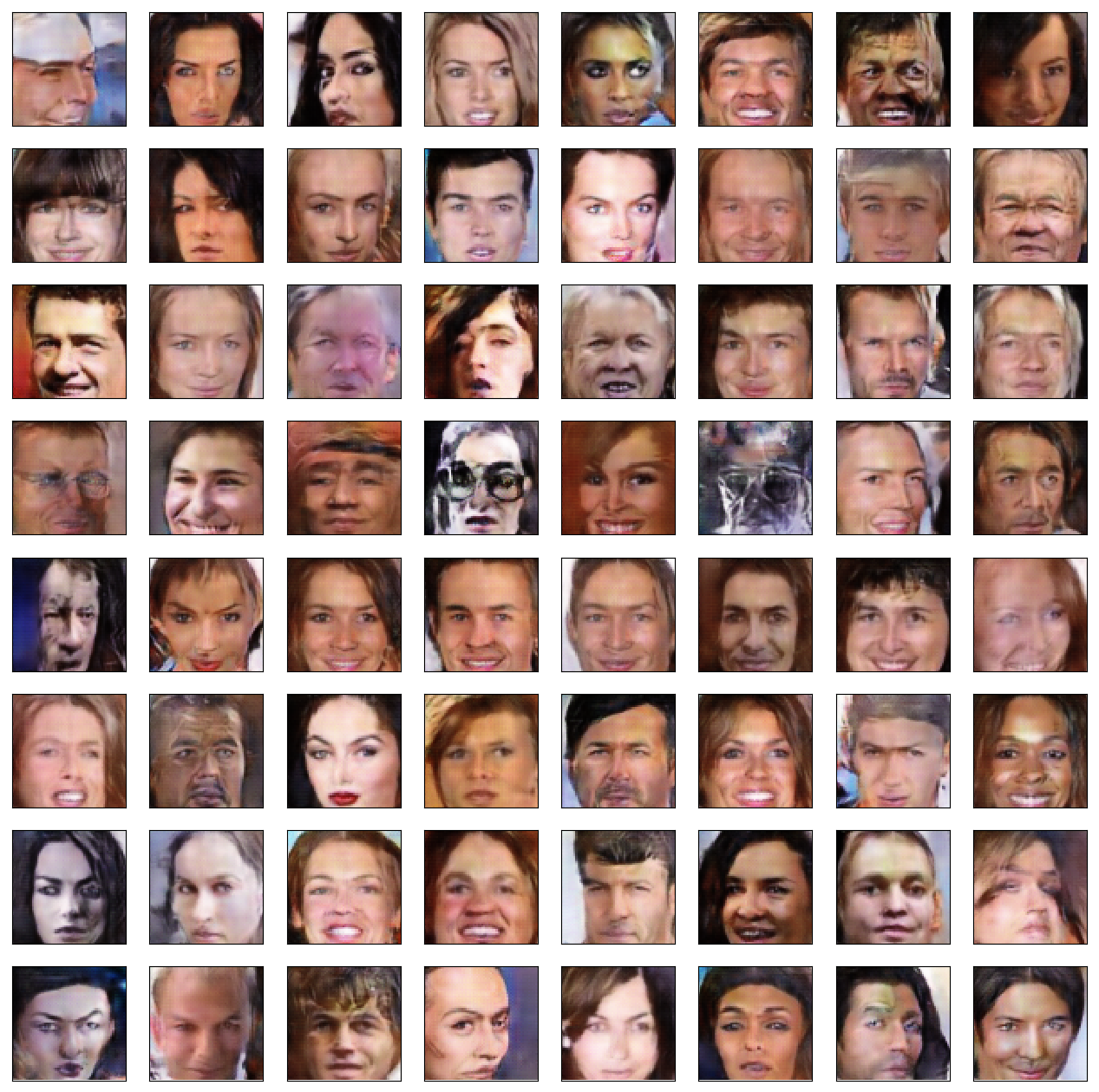} \\
{\small (a) Samples generated from the pre-trained VAE} & {\small (b) Samples generated from the pre-trained DCGAN}
\end{tabular}
\caption{Visualizations of samples generated from the pre-trained generative models.}
\label{fig:samples}
\end{center}
\end{figure}

\end{document}